\documentclass{article}

\usepackage{arxiv}
\usepackage{geometry}
\usepackage[authoryear,round]{natbib}
\usepackage{amsmath, amssymb}
\usepackage{amsthm}
\newtheorem{proposition}{Proposition}
\DeclareMathOperator*{\argmin}{arg\,min}
\usepackage{graphicx}
\usepackage{multirow}
\usepackage{tabularx}
\usepackage{booktabs}
\usepackage{caption}
\usepackage{float}
\usepackage[dvipsnames]{xcolor}
\usepackage{stmaryrd} 
\usepackage{algorithm,algpseudocode} 
\usepackage{bm,amsmath}              
\usepackage{wrapfig}
\usepackage[utf8]{inputenc} 
\usepackage[T1]{fontenc}    
\usepackage{hyperref}       
\usepackage{url}            
\usepackage{booktabs}       
\usepackage{amsfonts}       
\usepackage{nicefrac}       
\usepackage{microtype}      
\usepackage{lipsum}
\usepackage{graphicx}
\graphicspath{ {./images/} }

\title{Low-Rank and Structured Sparse Tensor Decomposition for Anomaly Detection in Multivariate Functional Data}

\author{
  Mohammad N. Bisheh\thanks{These authors contributed equally to this work.} \\
  School of Industrial and Systems Engineering\\
  Georgia Institute of Technology\\
  Atlanta, GA\\
  \And
  Che-Yi Liao\footnotemark[1] \\
  School of Industrial and Systems Engineering\\
  Georgia Institute of Technology\\
  Atlanta, GA\\
  \And
  Kamran Paynabar \\
  Stewart School of Industrial and Systems Engineering\\
  Georgia Institute of Technology\\
  Atlanta, GA
}

\begin{document}
\maketitle
\begin{abstract}
Multivariate functional data arise in many modern manufacturing systems, where multiple sensors record densely sampled process trajectories. Monitoring such data is challenging because nominal variation is strongly correlated across samples, sensors, and time, while faults may appear either as isolated deviations or as structured departures concentrated within a limited number of sensor-specific temporal trajectories. We propose two unsupervised sparse tensor decomposition methods that preserve this multimode structure. Entrywise Sparse CP Decomposition (ES-CP) uses an entrywise \(\ell_1\) penalty to identify localized anomalies, whereas Fiberwise Sparse-Group Lasso CP Decomposition (FG-Lasso) combines entrywise and fiberwise penalties to detect both localized deviations and anomalies concentrated within temporal fibers. Both methods represent nominal process behavior through a low-rank CP decomposition and are estimated using alternating optimization with closed-form sparse-component updates. Two simulation studies evaluate performance under different fault structures, signal severities, noise levels, and missing observations. FG-Lasso attains or ties the highest macro F$_1$ score in almost all settings in the first study and achieves the highest macro F$_1$ score. In a multichannel forging-process case study, FG-Lasso and ES-CP obtain macro F$_1$ scores of 0.85 and 0.82, respectively, compared with 0.69 or lower for TRPCA and PCA-based anomaly detectors. The results demonstrate that explicitly matching the sparse penalty to the anticipated fault structure improves both anomaly detection and fault localization in high-dimensional functional processes.
\end{abstract}


\section{Introduction}
\label{sec:intro}

Functional data analysis (FDA) and profile monitoring have gained significant attention due to advances in sensor technologies generating large volumes of continuous data across time or space. Functional data consists of measurements taken over a continuous domain, such as time or space, resulting in data represented by curves or surfaces rather than discrete points \citep{ramsay2005principal, kokoszka2017introduction}. This contrasts with traditional scalar data, which capture only single-point measurements, limiting their ability to fully reflect dynamic processes. Consequently, functional data offer richer information, enabling early detection and diagnosis of subtle process anomalies that would otherwise go unnoticed. However, this advantage comes with analytical challenges due to the inherent high-dimensionality, complexity, and correlation present in functional data.

Anomaly detection is a critical requirement in modern manufacturing and industrial processes, as it enables early identification and prevention of faults, ensuring product quality and operational reliability. In recent years, significant advances have been made in imrpoving the detection and diagnosis of anomalies in complex and high-dimensional (HD) data streams. These improvements range from change-point detection methods tailored to multivariate profiles, such as those described by \citep{paynabar2016change}, to robust statistical approaches specifically designed for functional data \citep{capezza2024robust}. Such methods address various forms of process anomalies by accurately detecting shifts or deviations in the underlying functional structures, making it possible to mitigate issues before they result in costly downtime or quality degradation.

A particularly illustrative example of functional data arises in a multi-operation forging press process described in previous studies by \cite{paynabar2016change} and \cite{yan2018real}. As shown in Figure \ref{fig:tonnageintro}, this forging press has four strain gauges mounted on each of its columns, continuously recording force measurements throughout a full $360^{\circ}$ crank angle cycle. Each cycle generates a set of four detailed curves that collectively form a multivariate functional profile. These profiles encode critical process information, reflecting die conditions, load balance, and material properties, thus enabling the timely detection of anomalies such as die wear or material inconsistency. Detecting these anomalies early prevents costly downtime and ensures consistent product quality and process reliability.

\vspace{10pt}
\begin{figure}[htbp]
  \centering
  \includegraphics[width= 0.8\linewidth]{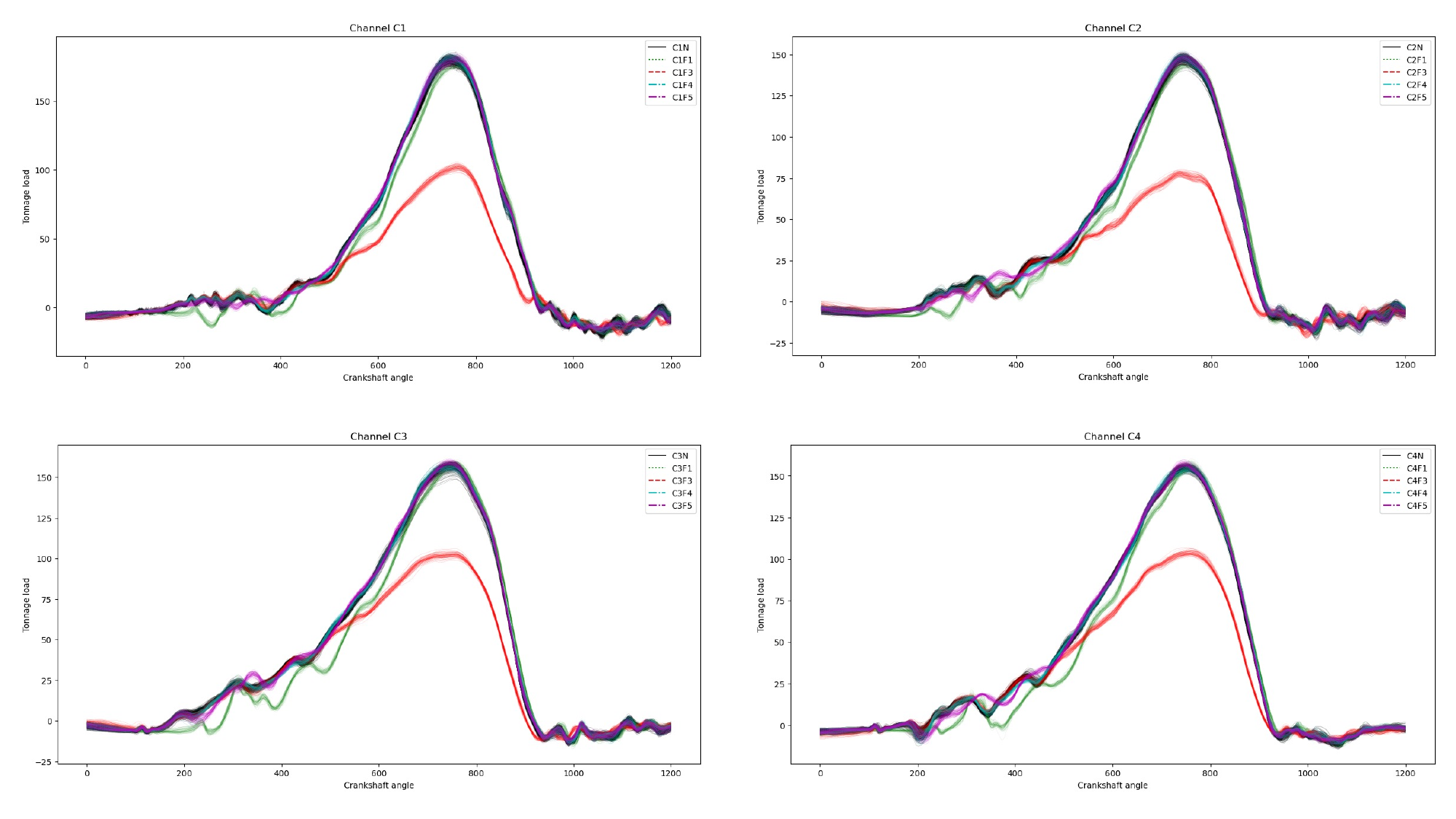}
  \caption{\scriptsize Overlay of four-channel tonnage profiles from a multi-operation forging press. Black curves denote healthy strokes (C1N); colored curves correspond to five engineered fault modes (C1F1--C1F5).}
  \label{fig:tonnageintro}
\end{figure}
\vspace{10pt} 

Analyzing multivariate functional data involves observing multiple correlated variables continuously or discretely across space or time. This type of data presents significant challenges in process monitoring. Traditional statistical process control (SPC) methods, which were originally developed for scalar or low-dimensional settings, are inadequate when applied to these complex, high-dimensional (HD) data streams. This inadequacy arises primarily from their inability to effectively handle intricate correlations and dependencies inherent in multivariate functional data, often leading to reduced sensitivity in detecting anomalies and elevated false alarm rates \citep{menafoglio2018profile, gomez2020functional, centofanti2021functional}. Although conventional dimension reduction approaches such as functional principal component analysis (FPCA) can mitigate some of these challenges by summarizing primary modes of variation, they still face critical limitations. Specifically, FPCA tends to obscure localized anomalies because such anomalies can affect multiple principal components simultaneously, making their detection challenging amid normal variability \citep{capezza2024robust, fang2021multi, wang2025tensor}.

Recent methodological advancements have partially overcome these limitations. For example, tensor decomposition methods have emerged to address these analytical challenges more effectively. Unlike functional data represented by curves or surfaces, tensor data naturally extend this representation into higher dimensions, incorporating multiple correlated modes (e.g., sensors, time, spatial location) simultaneously. Thus, tensors explicitly preserve multi-dimensional structure and interactions among modes, making them particularly suitable for multivariate functional monitoring. Recent work such as \citep{gahrooei2021multiple} demonstrated the potential of tensor-based approaches to extract comprehensive features and between mode relationships. Tensor decomposition methods, such as CANDECOMP/PARAFAC (CP) and Tucker decompositions, capture essential structural characteristics and facilitate isolation of anomalous components, significantly enhancing anomaly detection capabilities.

Motivated by these limitations, we develop two methods for unsupervised monitoring of multivariate
functional observations. The first method, termed Entrywise Sparse CP Decomposition (ES-CP), models localized anomalies through an entrywise \(\ell_1\) penalty on the anomaly tensor. The second method, termed Fiberwise Sparse-Group Lasso CP Decomposition (FG-Lasso), combines entrywise and fiberwise penalties to detect deviations concentrated within a limited number of sample-sensor temporal trajectories. In both methods, the nominal process behavior is represented by a low-rank CP tensor, while the sparse component provides direct localization of anomalous samples, sensors, and time points.

The main contributions of this study are threefold. First, we formulate two structured models that distinguish between localized entrywise anomalies and deviations concentrated within temporal fibers.
Second, we develop alternating computational procedures with closed-form updates for the corresponding sparse components. Third, we evaluate the
methods through two simulation studies and a multichannel forging-process case study, comparing them with TRPCA and three PCA-based anomaly detectors:
Isolation Forest, Local Outlier Factor, and One-Class SVM. The results show that ES-CP is particularly effective for localized faults, whereas FG-Lasso
provides the strongest overall performance when anomalies exhibit structured dependence across temporal fibers.

The remainder of the paper is organized as follows.
Section~\ref{sec:lite} reviews the literature on functional profile monitoring, tensor decomposition, and robust anomaly detection. Section~\ref{sec:method} introduces the proposed ES-CP and FG-Lasso models and presents their corresponding optimization procedures. Section~\ref{sec:Simulation} evaluates the methods through two simulation studies, and Section~\ref{sec:Case} presents the multichannel forging-process case study. Finally, Section~\ref{sec:Conclusion} summarizes the main findings and discusses directions for future research.

\section{Literature Review}
\label{sec:lite}

Advancements in sensor technology have significantly increased the complexity and dimensionality of monitoring data generated by manufacturing and industrial processes. Consequently, sophisticated methods are required to manage and analyze these data effectively. Classical statistical process monitoring (SPM) methods, such as Hotelling’s $T^2$ \citep{oyegoke2025hotelling}, multivariate exponentially weighted moving average (MEWMA) \citep{sun2024ewma, fang2022online}, and multivariate cumulative sum (MCUSUM) \citep{ye2022generic, ye2024online} charts, historically provided effective solutions for scalar or low-dimensional multivariate data \citep{perdikis2019survey, ahmed2025comprehensive}. However, these methods have significant limitations when dealing with complex, high-dimensional, and correlated functional data streams, as they cannot adequately capture structural dependencies across spatial or temporal dimensions \citep{jalilibal2024control}. To partially address the dimensionality challenge, traditional dimension-reduction techniques such as functional principal component analysis (FPCA) have been widely utilized. Nevertheless, FPCA lacks robustness to outliers, which can dominate the principal component space and obscure critical process variations \citep{capezza2024robust, maronna2019robust}.

To overcome limitations associated with high-dimensionality, recent works have introduced sampling-based monitoring frameworks. \citep{nabhan2021correlation} proposed a correlation-based dynamic sampling method, which adaptively selects subsets of variables based on their correlation structures and monitoring statistics contributions. This approach significantly reduces sampling costs and computational burdens by exploiting redundancy among correlated sensors. However, it relies heavily on accurate correlation matrix estimation and operates under the assumption of an existing monitoring model, without explicitly modeling or separating smooth signals from anomalous components. Furthermore, it does not inherently detect structured or temporally localized sparse anomalies common in real-time monitoring scenarios.

Sparse decomposition methods, particularly those utilizing $l_1$-norm and fused Lasso penalties, have been effective at isolating sparse anomalies from smooth backgrounds. For instance, \citep{yan2018real} introduced a spatio-temporal smooth sparse decomposition (ST-SSD) method for real-time monitoring of high-dimensional tensor streams, demonstrating enhanced detection capability for spatial and temporal anomalies in manufacturing processes. Similarly, \citep{yang2025anomaly} developed a constrained low-rank plus sparse decomposition tailored to photovoltaic (PV) image data, effectively identifying defect regions against smooth solar panel backgrounds through physically motivated constraints. Additionally, to handle high computational demands and data transfer costs associated with high-resolution image streams, \citep{mou2023compressed} proposed Compressed Smooth Sparse Decomposition (CSSD), embedding the sparse-smooth model within a compressive sensing framework to enable lightweight and rapid anomaly detection.

Tensor decomposition techniques, including CP and Tucker decompositions, further enhance the analysis of multi-dimensional functional data by explicitly preserving structural information across multiple modes. Recent developments such as functional tensor singular value decomposition (SVD) \citep{han2024guaranteed} and tensor-on-tensor regression \citep{gahrooei2021multiple} have expanded analytical capabilities, enabling predictive modeling and robust feature extraction from heterogeneous data sources. By explicitly maintaining multi-mode data structures, tensor methods mitigate limitations inherent in vectorized approaches, significantly improving interpretability and accuracy in anomaly detection tasks.

Incorporating structural assumptions into anomaly detection methods through regularization techniques such as Group Lasso and fused Lasso further enhances interpretability and effectiveness. For example, \citep{li2025real} introduced a spatio-temporal weighted graphical Lasso approach that captures temporal consistency and improves anomaly localization in dynamic tensor data. Similarly, \citep{qiu2023spatio} proposed spatial Lasso-based monitoring methods that enforce smoothness across sensor networks, thereby increasing detection sensitivity and interpretability. Despite these advantages, such methods typically rely on predefined graph structures or strong smoothness assumptions, potentially limiting their adaptability to evolving or unknown industrial processes.

Several additional studies have addressed anomaly detection through decomposition-based methods tailored to high-dimensional spatio-temporal data. \citep{shen2022robust} developed the Smooth Sparse Robust Tensor Decomposition (SS-RTD) model for separating background, foreground, and noise components in noisy videos from additive manufacturing processes. This method explicitly models background dynamics via low-rank Tucker decomposition, foreground structures through spatio-temporal smoothness, and noise using $l_1$-norm sparsity, achieving state-of-the-art denoising and segmentation performance. Complementarily, \citep{gomez2022adaptive} introduced an adaptive sampling framework that integrates smooth-sparse decomposition with a correlation-based sampling strategy for efficient online monitoring and diagnosis of high-dimensional data streams. While highly effective in reducing monitoring costs and improving computational scalability, this approach assumes pre-identified informative sensors and primarily focuses on spatial sparsity. Lastly, the Sparse Group Lasso Principal Component Analysis (SGL-PCA) developed by \citep{wang2022sgl} constructs health indices for mixed high-dimensional signals using structured group Lasso penalties and monotonicity constraints. Although particularly effective for health progression monitoring and residual life prediction, the SGL-PCA method emphasizes gradual, monotonic degradation patterns. Consequently, it may not generalize effectively to processes characterized by abrupt or localized faults.

Despite significant progress, existing methods rarely unify smoothness, sparsity, and interpretability within a single tensor-based monitoring framework. To address to critical gaps identified in existing methods for monitoring HD functional data, this paper proposes two robust and innovative tensor decomposition frameworks of ES-CP and FG-Lasso that effectively integrate low rank smoothness and structured sparsity for enhanced anomaly detection. Unlike traditional statistical process monitoring approaches, which often struggle to adequately handle HD correlations and are sensitive to outliers, our methods explicitly factorize observed functional data into a low rank tensor representing stable process dynamics and a sparse tensor isolating anomalies. The ES-CP model employs an entry-wise Lasso penalty tailored to pinpoint isolated anomalies such as brief spikes and measurement glitches, while the FG-Lasso extends this capability by incorporating a group-wise penalty that efficiently detects temporally or spatially extended anomalies affecting multiple sensors or time intervals. Through an alternating optimization strategy combining CP decomposition with closed form proximal updates, our methods achieve robust performance even in scenarios characterized by significant noise, missing data, and the absence of labeled training anomalies. Consequently, these proposed approaches significantly advance the SOTA in functional data monitoring by offering superior sensitivity to subtle yet critical anomalies, clearer interpretability of detected faults, and improved practical utility for real-time monitoring tasks within complex industrial processes.

\section{Method}\label{sec:method}

This section provides a detailed description of our proposed robust tensor decomposition approach designed to effectively analyze multivariate functional data. We first outline essential tensor notations and algebraic operations used throughout our formulation. Subsequently, we introduce two models, the ES-CP and the FG-Lasso. Finally, we detail the computational procedures to solve each model efficiently.

\subsection{Basic Notation and Multilinear Algebra}
\label{subsec:tensor_notation}

Let $\mathcal{X}\in\mathbb{R}^{I\times J\times K}$ denote the observed third-order tensor, where \(I\) is the number of samples, \(J\) is the number of sensors or channels, and \(K\) is the number of discretized time points. The entry \(X_{ijk}\) denotes the observation from sample \(i\), sensor \(j\), and time point \(k\). Accordingly, \(\mathcal{X}_{i::}\in\mathbb{R}^{J\times K}\) denotes the complete multivariate functional profile associated with sample \(i\), while \(\mathcal{X}_{ij:}\in\mathbb{R}^{K}\) denotes the temporal fiber associated with sample \(i\) and sensor \(j\).

Throughout the paper, tensors are denoted by calligraphic capital letters, such as \(\mathcal{X}\), matrices by bold capital letters, such as \(\mathbf{X}\), vectors by bold lowercase letters, such as \(\mathbf{x}\), and scalars by lowercase letters, such as \(x\). For a third-order tensor \(\mathcal{X}\), the Frobenius norm and entrywise \(\ell_1\) norm are defined, respectively, as
\begin{equation}
    \|\mathcal{X}\|_F
    =
    \left(
    \sum_{i=1}^{I}
    \sum_{j=1}^{J}
    \sum_{k=1}^{K}
    X_{ijk}^{2}
    \right)^{1/2},
    \qquad
    \|\mathcal{X}\|_1
    =
    \sum_{i=1}^{I}
    \sum_{j=1}^{J}
    \sum_{k=1}^{K}
    |X_{ijk}|.
    \label{eq:tensor_norms}
\end{equation}
The inner product between two tensors
\(\mathcal{X},\mathcal{Y}\in\mathbb{R}^{I\times J\times K}\) is
\begin{equation}
    \langle\mathcal{X},\mathcal{Y}\rangle
    =
    \sum_{i=1}^{I}
    \sum_{j=1}^{J}
    \sum_{k=1}^{K}
    X_{ijk}Y_{ijk},
\end{equation}
so that
\(\|\mathcal{X}\|_F^2=\langle\mathcal{X},\mathcal{X}\rangle\).
For a temporal fiber \(\mathcal{X}_{ij:}\), its Euclidean norm is denoted by
\[
    \|\mathcal{X}_{ij:}\|_2
    =
    \left(
    \sum_{k=1}^{K}X_{ijk}^{2}
    \right)^{1/2}.
\]

The outer product of three vectors
\(\mathbf{a}\in\mathbb{R}^{I}\),
\(\mathbf{b}\in\mathbb{R}^{J}\), and
\(\mathbf{c}\in\mathbb{R}^{K}\) is denoted by
\(\mathbf{a}\circ\mathbf{b}\circ\mathbf{c}\) and produces a rank-one tensor whose entries are
\begin{equation}
    \left(
    \mathbf{a}\circ\mathbf{b}\circ\mathbf{c}
    \right)_{ijk}
    =
    a_i b_j c_k.
\end{equation}
The Kronecker product is denoted by \(\otimes\). For matrices
\(\mathbf{A}\in\mathbb{R}^{I\times R}\) and \(\mathbf{B}\in\mathbb{R}^{J\times R}\), their Khatri--Rao product is the columnwise Kronecker product
\begin{equation}
    \mathbf{A}\odot\mathbf{B}
    =
    \begin{bmatrix}
        \mathbf{a}_1\otimes\mathbf{b}_1 &
        \cdots &
        \mathbf{a}_R\otimes\mathbf{b}_R
    \end{bmatrix}
    \in\mathbb{R}^{IJ\times R},
    \label{eq:khatri_rao}
\end{equation}
where \(\mathbf{a}_r\) and \(\mathbf{b}_r\) are the \(r\)th columns of \(\mathbf{A}\) and \(\mathbf{B}\), respectively. The elementwise product of two matrices of the same dimension is denoted by \(\ast\). A useful identity for the Khatri--Rao product is
\begin{equation}
    (\mathbf{A}\odot\mathbf{B})^\top
    (\mathbf{A}\odot\mathbf{B})
    =
    (\mathbf{A}^\top\mathbf{A})
    \ast
    (\mathbf{B}^\top\mathbf{B}).
    \label{eq:khatri_rao_gram}
\end{equation}

Let \(I_1=I\), \(I_2=J\), and \(I_3=K\). The mode-\(n\) unfolding, or matricization, of \(\mathcal{X}\) is denoted by
\begin{equation}
    \mathbf{X}_{(n)}
    \in
    \mathbb{R}^{I_n\times\prod_{m\neq n}I_m},
\end{equation}
and arranges the mode-\(n\) fibers of \(\mathcal{X}\) as columns of a matrix. For a tensor
\(\mathcal{T}\in\mathbb{R}^{I_1\times\cdots\times I_N}\) and a matrix
\(\mathbf{U}\in\mathbb{R}^{L\times I_n}\), the mode-\(n\) product
\[
    \mathcal{T}'=\mathcal{T}\times_n\mathbf{U}
\]
produces a tensor of dimension
\[
    I_1\times\cdots\times I_{n-1}\times L
    \times I_{n+1}\times\cdots\times I_N.
\]
Its mode-\(n\) unfolding satisfies
\begin{equation}
    \mathbf{T}'_{(n)}
    =
    \mathbf{U}\mathbf{T}_{(n)}.
    \label{eq:mode_n_product}
\end{equation}

The nominal process structure is represented by a low-rank tensor \(\mathcal{L}\). We approximate \(\mathcal{L}\) using a rank-\(R\) CP decomposition:
\begin{equation}
\begin{aligned}
    \mathcal{L}
    &=
    \llbracket
    \boldsymbol{\lambda};
    \mathbf{A}^{(1)},
    \mathbf{A}^{(2)},
    \mathbf{A}^{(3)}
    \rrbracket \\
    &=
    \sum_{r=1}^{R}
    \lambda_r\,
    \mathbf{a}^{(1)}_r
    \circ
    \mathbf{a}^{(2)}_r
    \circ
    \mathbf{a}^{(3)}_r,
\end{aligned}
\label{eq:cp_full}
\end{equation}
where
\[
    \boldsymbol{\lambda}
    =
    (\lambda_1,\ldots,\lambda_R)^\top
    \in\mathbb{R}_{+}^{R}
\]
contains the component weights, and
\[
    \mathbf{A}^{(1)}\in\mathbb{R}^{I\times R},
    \qquad
    \mathbf{A}^{(2)}\in\mathbb{R}^{J\times R},
    \qquad
    \mathbf{A}^{(3)}\in\mathbb{R}^{K\times R}
\]
are the factor matrices associated with the sample, sensor, and temporal modes, respectively. The vector \(\mathbf{a}^{(n)}_r\) denotes the \(r\)th column of \(\mathbf{A}^{(n)}\). To address the scale indeterminacy of the CP representation, the factor vectors are normalized such that
\[
    \|\mathbf{a}^{(n)}_r\|_2=1,
    \qquad n=1,2,3,
\]
and the corresponding component scales are absorbed into \(\lambda_r\).

Let
\[
    \boldsymbol{\Lambda}
    =
    \operatorname{diag}(\lambda_1,\ldots,\lambda_R).
\]
The mode-wise matricized forms of the CP decomposition are
\begin{align}
    \mathbf{L}_{(1)}
    &=
    \mathbf{A}^{(1)}
    \boldsymbol{\Lambda}
    \left(
    \mathbf{A}^{(3)}
    \odot
    \mathbf{A}^{(2)}
    \right)^\top,
    \label{eq:cp_mode1}
    \\
    \mathbf{L}_{(2)}
    &=
    \mathbf{A}^{(2)}
    \boldsymbol{\Lambda}
    \left(
    \mathbf{A}^{(3)}
    \odot
    \mathbf{A}^{(1)}
    \right)^\top,
    \label{eq:cp_mode2}
    \\
    \mathbf{L}_{(3)}
    &=
    \mathbf{A}^{(3)}
    \boldsymbol{\Lambda}
    \left(
    \mathbf{A}^{(2)}
    \odot
    \mathbf{A}^{(1)}
    \right)^\top.
    \label{eq:cp_mode3}
\end{align}
These matricized representations transform the estimation of each CP factor matrix into a matrix least-squares problem and form the basis of the alternating optimization algorithms developed in the following sections.

\subsection{Model I: Entrywise Sparse CP Decomposition}
\label{subsec:escp}

The first proposed model, referred to as ES-CP, is intended for faults that manifest as localized
departures at a relatively small number of sample-sensor-time locations. Examples include transient spikes, measurement glitches, short-duration shifts, and other localized deviations from the nominal functional profile. We assume that the observed tensor admits the decomposition
\begin{equation}
    \mathcal X
    =
    \mathcal L
    +
    \mathcal S
    +
    \mathcal E,
    \label{eq:escp_decomposition}
\end{equation}
where \(\mathcal L\) represents the nominal low-rank process structure, \(\mathcal S\) contains anomalous deviations, and \(\mathcal E\) represents unstructured measurement noise and model approximation error. The nominal component \(\mathcal L\) is represented by the \(R\)-component CP model in \eqref{eq:cp_full}, whereas \(\mathcal S\) is assumed to be sparse at the entry level.

For a given CP rank \(R\), the ES-CP estimator is defined by
\begin{equation}
\begin{aligned}
&
\argmin_{
    \mathcal S,\,
    \mathbf A^{(1)},\,
    \mathbf A^{(2)},\,
    \mathbf A^{(3)}
}
\quad
\left\|
    \mathcal X-\mathcal L-\mathcal S
\right\|_F^2
+
\alpha_1\|\mathcal S\|_1
+
\alpha_2
\left\|
    \mathbf A^{(3)}
\right\|_F^2
\\
&\text{s.t. }
\quad
\mathcal L
=
\llbracket
    \boldsymbol\lambda;
    \mathbf A^{(1)},
    \mathbf A^{(2)},
    \mathbf A^{(3)}
\rrbracket,
\end{aligned}
\label{eq:escp_objective}
\end{equation}

where \(\alpha_1>0\) controls entrywise sparsity and \(\alpha_2\geq0\) regularizes the temporal factor matrix. The estimated nominal component is subsequently reconstructed as \begin{equation}
    \widehat{\mathcal L}
    =
    \llbracket
        \widehat{\mathbf A}^{(1)},
        \widehat{\mathbf A}^{(2)},
        \widehat{\mathbf A}^{(3)}
    \rrbracket.
    \label{eq:escp_estimated_lowrank}
\end{equation}

The reconstruction term in \eqref{eq:escp_objective} measures the discrepancy between the observed tensor and its low-rank-plus-sparse representation. The entrywise penalty $\|\mathcal S\|_1$ encourages most entries of the anomaly tensor to be exactly zero. Consequently, a nonzero value of \(\widehat S_{ijk}\) identifies a localized departure for sample \(i\), sensor \(j\), and time point \(k\). The parameter \(\alpha_1\) determines the degree of sparsity in \(\widehat{\mathcal S}\). Larger values require stronger residual deviations to be retained as anomalies, whereas smaller values increase sensitivity to weaker departures. The parameter \(\alpha_2\) shrinks the magnitude of the temporal factor matrix and stabilizes estimation of the low-rank component. This ridge term does not directly impose temporal smoothness or continuity; rather, it limits the extent to which the temporal factors can absorb irregular variation. The CP rank \(R\) controls the complexity of the nominal component and is selected from a prespecified candidate set. Larger values permit a more flexible representation of the nominal process but may also allow the low-rank component to absorb anomalous variation. Smaller values provide a more parsimonious representation but may leave systematic nominal structure in the residual tensor.

The formulation estimates the nominal and anomalous components jointly. The low-rank component captures the dominant dependence across samples, sensors, and time, while the sparse component isolates localized departures from that structure. The objective is not jointly convex because of the multilinear CP representation, but it admits an alternating optimization scheme in which the sparse tensor and CP factor matrices are updated conditionally. The computational procedure is presented in the following subsection.

\subsection{Optimization of the ES-CP Model}
\label{subsec:escp_optimization}

The ES-CP criterion in \eqref{eq:escp_objective} is not jointly convex because the CP representation is multilinear in its factor matrices. We therefore use an alternating computational procedure. At each outer iteration, the sparse tensor from the preceding iteration is held fixed while an \(R\)-component CP approximation is fitted to the anomaly-adjusted tensor. The sparse component is then updated by applying entrywise soft-thresholding to the resulting residual. For the general CP representation, component scales may be separated from the factor matrices through the weight vector \(\boldsymbol\lambda\). In the implementation considered in this study, however, the weights are fixed at one and the factors are not column-normalized; hence, all component scales remain absorbed into the factor matrices.
For a fixed low-rank estimate
\begin{equation}
    \mathcal L
    =
    \llbracket
        \mathbf A^{(1)},
        \mathbf A^{(2)},
        \mathbf A^{(3)}
    \rrbracket,
\end{equation}
define the residual tensor
\begin{equation}
    \mathcal R
    =
    \mathcal X-\mathcal L.
    \label{eq:escp_residual}
\end{equation}
The sparse-component subproblem associated with
\eqref{eq:escp_objective} is
\begin{equation}
    \mathcal S^\star
    =
    \argmin_{\mathcal S}
    \left\{
        \|\mathcal R-\mathcal S\|_F^2
        +
        \alpha_1\|\mathcal S\|_1
    \right\}.
    \label{eq:escp_s_subproblem}
\end{equation}
Because \eqref{eq:escp_s_subproblem} is separable across tensor entries, it reduces to independent scalar optimization problems.

\begin{proposition}
\label{prop:escp_soft}
For fixed \(\mathcal L\), the unique minimizer of \eqref{eq:escp_s_subproblem} is obtained by entrywise soft-thresholding:
\begin{equation}
    S_{ijk}^{\star}
    =
    \operatorname{soft}_{\alpha_1/2}(R_{ijk})
    =
    \operatorname{sign}(R_{ijk})
    \left(
        |R_{ijk}|-\frac{\alpha_1}{2}
    \right)_+,
    \label{eq:escp_soft}
\end{equation}
where \((x)_+=\max(x,0)\).
\end{proposition}

The proof is provided in Appendix~\ref{app:proof_elcp_soft_threshold}. 

Residual entries whose magnitudes do not exceed \(\alpha_1/2\) are set to zero, whereas larger entries are retained after shrinkage. Thus, increasing \(\alpha_1\) produces a sparser anomaly estimate and requires a larger residual deviation for an entry to be identified as anomalous.

For fixed \(\mathcal S\), define the anomaly-adjusted tensor
\begin{equation}
    \mathcal Y
    =
    \mathcal X-\mathcal S.
    \label{eq:escp_adjusted_tensor}
\end{equation}
The implementation estimates a rank-\(R\) CP representation of \(\mathcal Y\) using alternating Tikhonov-regularized normal-equation updates. These updates provide numerically stabilized approximations to the conditional CP least-squares solutions. Under the unfolding convention used in the implementation, define
\begin{equation}
\begin{aligned}
    \mathbf Z_1
    &=
    \mathbf A^{(2)}
    \odot
    \mathbf A^{(3)},\\
    \mathbf Z_2
    &=
    \mathbf A^{(1)}
    \odot
    \mathbf A^{(3)},\\
    \mathbf Z_3
    &=
    \mathbf A^{(1)}
    \odot
    \mathbf A^{(2)}.
\end{aligned}
\label{eq:escp_khatri_rao_matrices}
\end{equation}
For each mode, let
\begin{equation}
    \mathbf B_n
    =
    \mathbf Y_{(n)}\mathbf Z_n,
    \qquad
    \mathbf V_n
    =
    \mathbf Z_n^\top\mathbf Z_n.
    \label{eq:escp_BV_definitions}
\end{equation}
Using the Khatri--Rao Gram identity,
\begin{equation}
\begin{aligned}
    \mathbf V_1
    &=
    \left(
        \mathbf A^{(2)\top}\mathbf A^{(2)}
    \right)
    \ast
    \left(
        \mathbf A^{(3)\top}\mathbf A^{(3)}
    \right),\\
    \mathbf V_2
    &=
    \left(
        \mathbf A^{(1)\top}\mathbf A^{(1)}
    \right)
    \ast
    \left(
        \mathbf A^{(3)\top}\mathbf A^{(3)}
    \right),\\
    \mathbf V_3
    &=
    \left(
        \mathbf A^{(1)\top}\mathbf A^{(1)}
    \right)
    \ast
    \left(
        \mathbf A^{(2)\top}\mathbf A^{(2)}
    \right).
\end{aligned}
\label{eq:escp_gram_matrices}
\end{equation}

For the first and second modes, the implementation applies the numerically regularized updates
\begin{equation}
    \mathbf A^{(1)}
    \leftarrow
    \mathbf B_1
    \left(
        \mathbf V_1^\top\mathbf V_1
    \right)^{\dagger}
    \mathbf V_1^\top,
    \label{eq:escp_A1_update}
\end{equation}
and
\begin{equation}
    \mathbf A^{(2)}
    \leftarrow
    \mathbf B_2
    \left(
        \mathbf V_2^\top\mathbf V_2
    \right)^{\dagger}
    \mathbf V_2^\top,
    \label{eq:escp_A2_update}
\end{equation}
where \(\epsilon>0\) is a small numerical constant and
\(\dagger\) denotes the Moore--Penrose pseudoinverse.

The temporal factor matrix is updated as
\begin{equation}
    \mathbf A^{(3)}
    \leftarrow
    \mathbf B_3
    \left(
        \mathbf V_3^\top\mathbf V_3
        +
        \alpha_2\mathbf I_R
    \right)^{\dagger}
    \mathbf V_3^\top.
    \label{eq:escp_A3_update}
\end{equation}
The parameter \(\alpha_2\) is used in the numerical update of the temporal factor matrix to provide shrinkage and improve conditioning. Because the implemented CP factors are not column-normalized, this regularization is interpreted relative to the scaling generated by the alternating updates, rather than as a scale-invariant penalty on the reconstructed low-rank tensor.

After completing the prescribed number of inner iterations, the low-rank tensor is reconstructed as
\begin{equation}
    \mathcal L
    =
    \llbracket
        \mathbf A^{(1)},
        \mathbf A^{(2)},
        \mathbf A^{(3)}
    \rrbracket.
    \label{eq:escp_lowrank_reconstruction}
\end{equation}

The complete ES-CP procedure initializes \(\mathcal S^{(0)}=\mathbf 0\) and alternates between the low-rank and sparse updates. At outer iteration \(t\), Algorithm~\ref{alg:escp_cpals} is applied to the adjusted tensor
\[
    \mathcal Y^{(t)}
    =
    \mathcal X-\mathcal S^{(t-1)}
\]
to obtain the updated low-rank component
\[
    \mathcal L^{(t)}
    =
    \llbracket
        \mathbf A^{(1,t)},
        \mathbf A^{(2,t)},
        \mathbf A^{(3,t)}
    \rrbracket.
\]
The sparse component is then updated by
\begin{equation}
    \mathcal S^{(t)}
    =
    \operatorname{soft}_{\alpha_1/2}
    \left(
        \mathcal X-\mathcal L^{(t)}
    \right).
    \label{eq:escp_outer_S_update}
\end{equation}
The procedure is repeated for a prespecified number of outer iterations. The regularized CP-ALS procedure for updating $\mathcal L$ with $\mathcal S$ fixed is summarized in Algorithm~\ref{alg:escp_cpals}.

\begin{minipage}{0.85\textwidth}
\begin{algorithm}[H]
\raggedright
\fontsize{9pt}{10pt}\selectfont
\caption{Regularized CP update for the ES-CP low-rank component}
\label{alg:escp_cpals}

\hspace*{\algorithmicindent}
\textbf{Input:}
Adjusted tensor
\(\mathcal Y=\mathcal X-\mathcal S\),
rank \(R\), regularization parameter \(\alpha_2\),
numerical constant \(\epsilon\), and maximum number of iterations
\(T_{\mathrm{in}}\).

\begin{algorithmic}[1]

\State Randomly initialize
\(\mathbf A^{(2)}\in\mathbb R^{J\times R}\) and
\(\mathbf A^{(3)}\in\mathbb R^{K\times R}\)

\For{\(q=1,\ldots,T_{\mathrm{in}}\)}

    \State
    \(\mathbf Z_1
    \gets
    \mathbf A^{(2)}
    \odot
    \mathbf A^{(3)}\)

    \State
    \(\mathbf V_1
    \gets
    \left(
        \mathbf A^{(2)\top}\mathbf A^{(2)}
    \right)
    \ast
    \left(
        \mathbf A^{(3)\top}\mathbf A^{(3)}
    \right)\)

    \State
    \(\displaystyle
    \mathbf A^{(1)}
    \gets
    \mathbf Y_{(1)}
    \mathbf Z_1
    \left(
        \mathbf V_1^\top\mathbf V_1
    \right)^{\dagger}
    \mathbf V_1^\top
    \)

    \State
    \(\mathbf Z_2
    \gets
    \mathbf A^{(1)}
    \odot
    \mathbf A^{(3)}\)

    \State
    \(\mathbf V_2
    \gets
    \left(
        \mathbf A^{(1)\top}\mathbf A^{(1)}
    \right)
    \ast
    \left(
        \mathbf A^{(3)\top}\mathbf A^{(3)}
    \right)\)

    \State
    \(\displaystyle
    \mathbf A^{(2)}
    \gets
    \mathbf Y_{(2)}
    \mathbf Z_2
    \left(
        \mathbf V_2^\top\mathbf V_2
    \right)^{\dagger}
    \mathbf V_2^\top
    \)

    \State
    \(\mathbf Z_3
    \gets
    \mathbf A^{(1)}
    \odot
    \mathbf A^{(2)}\)

    \State
    \(\mathbf V_3
    \gets
    \left(
        \mathbf A^{(1)\top}\mathbf A^{(1)}
    \right)
    \ast
    \left(
        \mathbf A^{(2)\top}\mathbf A^{(2)}
    \right)\)

    \State
    \(\displaystyle
    \mathbf A^{(3)}
    \gets
    \mathbf Y_{(3)}
    \mathbf Z_3
    \left(
        \mathbf V_3^\top\mathbf V_3
        +
        \alpha_2\mathbf I_R
    \right)^{\dagger}
    \mathbf V_3^\top
    \)

\EndFor

\State
\(\displaystyle
\mathcal L
\gets
\llbracket
    \mathbf A^{(1)},
    \mathbf A^{(2)},
    \mathbf A^{(3)}
\rrbracket
\)

\State \Return \(\mathcal L\) and
\(\{\mathbf A^{(1)},\mathbf A^{(2)},\mathbf A^{(3)}\}\)

\end{algorithmic}
\end{algorithm}
\end{minipage}

\subsection{Model II: Fiberwise Sparse-Group Lasso CP Decomposition}
\label{subsec:fgcp}

The second proposed model, referred to as \emph{Fiberwise Sparse-Group Lasso CP Decomposition} (FG-Lasso), extends ES-CP to faults that exhibit structured deviations along a limited number of sample--sensor trajectories. Such faults may persist over multiple time points, affect a substantial portion of a sensor profile, or combine localized spikes with more distributed temporal departures. In these settings, entrywise sparsity alone does not exploit the tendency of anomalous entries to be concentrated within selected temporal fibers.

FG-Lasso uses the same low-rank-plus-sparse decomposition \eqref{eq:escp_decomposition} and the same \(R\)-component CP representation of the nominal tensor introduced in \eqref{eq:cp_full}. The distinction from ES-CP lies in the structural penalty imposed on the anomaly tensor. In addition to entrywise sparsity, FG-Lasso encourages group sparsity across the temporal fibers \(\mathcal S_{ij:}\).

Define the fiberwise mixed norm
\begin{equation}
    \|\mathcal S\|_{\mathrm{fib},2,1}
    =
    \sum_{i=1}^{I}
    \sum_{j=1}^{J}
    \|\mathcal S_{ij:}\|_2.
    \label{eq:fiber_mixed_norm}
\end{equation}
The Euclidean norm measures the aggregate anomaly magnitude within each sample--sensor trajectory, while summation across fibers encourages many complete temporal fibers to be estimated as zero.

For a given CP rank \(R\), the FG-Lasso estimator is defined by
\begin{equation}
\begin{aligned}
\min_{
    \mathbf A^{(1)},
    \mathbf A^{(2)},
    \mathbf A^{(3)},
    \mathcal S
}
\quad
&
\left\|
    \mathcal X-\mathcal L-\mathcal S
\right\|_F^2
\\
&\quad+
\alpha_1
\left[
    \gamma\|\mathcal S\|_1
    +
    (1-\gamma)
    \|\mathcal S\|_{\mathrm{fib},2,1}
\right]
+
\alpha_2
\left\|
    \mathbf A^{(3)}
\right\|_F^2,
\\
&\text{s.t. }
\quad
\mathcal L
=
\llbracket
    \boldsymbol\lambda;
    \mathbf A^{(1)},
    \mathbf A^{(2)},
    \mathbf A^{(3)}
\rrbracket,
\end{aligned}
\label{eq:fgcp_objective}
\end{equation}
where
\[
    \alpha_1\geq0,
    \qquad
    \alpha_2\geq0,
    \qquad
    0\leq\gamma\leq1.
\]

As in ES-CP, no separate CP-weight vector or factor-normalization constraints are introduced in the implemented parameterization. The two sparse penalties in \eqref{eq:fgcp_objective} operate at complementary structural levels. The entrywise term $\|\mathcal S\|_1$ encourages individual anomaly entries to be zero and preserves sensitivity to localized deviations. The fiberwise term $\|\mathcal S\|_{\mathrm{fib},2,1}$ performs group selection at the sample--sensor level. A temporal fiber may be set entirely to zero when its aggregate residual magnitude is insufficient to overcome the group penalty.

The parameter \(\alpha_1\) controls the overall strength of regularization applied to the anomaly tensor, while \(\alpha_2\) controls the ridge regularization applied to the temporal CP factor matrix \(\mathbf A^{(3)}\). The parameter \(\gamma\) determines how the sparse regularization is allocated between the entrywise and fiberwise penalties. Increasing \(\gamma\) places greater emphasis on entrywise sparsity, whereas decreasing \(\gamma\) places greater emphasis on selecting or eliminating complete sample--sensor temporal fibers.

The estimated anomaly tensor provides two levels of diagnostic information. A nonzero fiber \(\widehat{\mathcal S}_{ij:}\) identifies a sample--sensor trajectory associated with an anomalous departure, while its nonzero entries localize the time points at which the departure occurs. The fiberwise penalty does not directly impose temporal smoothness or contiguity within an active fiber. Rather, it encourages anomalous entries to be concentrated within a
limited number of sample--sensor trajectories. The optimization procedure and the closed-form fiberwise sparse-group update for \eqref{eq:fgcp_objective} are presented in the following subsection.

\subsection{Optimization of the FG-Lasso Model}
\label{subsec:fgcp_optimization}

The FG-Lasso criterion in \eqref{eq:fgcp_objective} is not jointly convex because the CP representation is multilinear in its factor matrices. We therefore use an alternating procedure. At each outer iteration, the sparse tensor from the preceding iteration is held fixed while a regularized rank-\(R\) CP approximation is fitted to the anomaly-adjusted tensor. The sparse component is then updated by applying a fiberwise sparse-group proximal operator to the resulting residual tensor.

The implemented parameterization does not normalize the CP factor matrices and does not introduce a separate vector of CP component weights.

For a fixed low-rank component, define
\begin{equation}
    \mathcal R
    =
    \mathcal X-\mathcal L.
    \label{eq:fgcp_residual}
\end{equation}

The sparse-tensor subproblem associated with
\eqref{eq:fgcp_objective} is
\begin{equation}
\begin{aligned}
    \mathcal S^\star
    =
    \argmin_{\mathcal S}
    \quad
    &
    \|\mathcal R-\mathcal S\|_F^2
    +
    \alpha_1\gamma\|\mathcal S\|_1
    \\
    &+
    \alpha_1(1-\gamma)
    \sum_{i=1}^{I}
    \sum_{j=1}^{J}
    \|\mathcal S_{ij:}\|_2.
\end{aligned}
\label{eq:fgcp_s_subproblem}
\end{equation}

Because the reconstruction loss and both penalties are separable across sample--sensor fibers, this problem can be written as
\begin{equation}
    \sum_{i=1}^{I}
    \sum_{j=1}^{J}
    \left[
        \|\mathcal R_{ij:}-\mathcal S_{ij:}\|_2^2
        +
        \alpha_1\gamma\|\mathcal S_{ij:}\|_1
        +
        \alpha_1(1-\gamma)\|\mathcal S_{ij:}\|_2
    \right].
    \label{eq:fgcp_fiber_separable}
\end{equation}

For a fixed fiber, let
\[
    \mathbf z=\mathcal R_{ij:},
    \qquad
    \mathbf s=\mathcal S_{ij:}.
\]
The corresponding optimization problem is
\begin{equation}
    \mathbf s^\star
    =
    \argmin_{\mathbf s\in\mathbb R^K}
    \left\{
        \|\mathbf z-\mathbf s\|_2^2
        +
        \alpha_1\gamma\|\mathbf s\|_1
        +
        \alpha_1(1-\gamma)\|\mathbf s\|_2
    \right\}.
    \label{eq:fgcp_fiber_subproblem}
\end{equation}

Define the entrywise and fiberwise thresholds as

\begin{equation}
    \tau_{\mathrm e}
    =
    \frac{\alpha_1\gamma}{2},
    \qquad
    \tau_{\mathrm g}
    =
    \frac{\alpha_1(1-\gamma)}{2}.
    \label{eq:fgcp_thresholds}
\end{equation}

For a vector \(\mathbf z\), define entrywise soft-thresholding by
\begin{equation}
    \left[
        \operatorname{soft}_{\tau_{\mathrm e}}(\mathbf z)
    \right]_k
    =
    \operatorname{sign}(z_k)
    \left(
        |z_k|-\tau_{\mathrm e}
    \right)_+,
    \qquad
    k=1,\ldots,K.
    \label{eq:fgcp_soft_operator}
\end{equation}

\begin{proposition}
\label{prop:fgcp_prox}
For fixed \(\mathcal L\), the sparse subproblem \eqref{eq:fgcp_s_subproblem} has a unique minimizer. For each temporal fiber, define
\begin{equation}
    \mathbf u_{ij}
    =
    \operatorname{soft}_{\tau_{\mathrm e}}
    \left(
        \mathcal R_{ij:}
    \right).
    \label{eq:fgcp_soft_fiber}
\end{equation}
Then
\begin{equation}
    \mathcal S_{ij:}^{\star}
    =
    \left(
        1-
        \frac{\tau_{\mathrm g}}
        {\|\mathbf u_{ij}\|_2}
    \right)_+
    \mathbf u_{ij},
    \label{eq:fgcp_compact_prox}
\end{equation}
with the convention that the right-hand side equals \(\mathbf 0\) when \(\mathbf u_{ij}=\mathbf 0\). Equivalently,
\begin{equation}
    \mathcal S_{ij:}^{\star}
    =
    \begin{cases}
        \mathbf 0,
        &
        \|\mathbf u_{ij}\|_2
        \leq
        \tau_{\mathrm g},
        \\[6pt]
        \displaystyle
        \left(
            1-
            \frac{\tau_{\mathrm g}}
            {\|\mathbf u_{ij}\|_2}
        \right)
        \mathbf u_{ij},
        &
        \|\mathbf u_{ij}\|_2
        >
        \tau_{\mathrm g}.
    \end{cases}
    \label{eq:fgcp_piecewise_prox}
\end{equation}
\end{proposition}

The proof is provided in Appendix~\ref{app:fgcp_prox_proof}. The update operates in two stages. First, entrywise soft-thresholding removes small deviations within the fiber. Second, group shrinkage either eliminates the complete soft-thresholded fiber or proportionally shrinks it. When \(\gamma=1\), the fiberwise penalty vanishes and \(\tau_{\mathrm g}=0\). The update therefore reduces to entrywise soft-thresholding with threshold \(\alpha_1/2\). When \(\gamma=0\), the entrywise penalty vanishes and \(\tau_{\mathrm e}=0\), yielding the pure fiberwise group-thresholding update
\begin{equation}
    \mathcal S_{ij:}^{\star}
    =
    \left(
        1-
        \frac{\alpha_1}
        {2\|\mathcal R_{ij:}\|_2}
    \right)_+
    \mathcal R_{ij:}.
    \label{eq:fgcp_group_only}
\end{equation}

The fiberwise update is summarized in Algorithm~\ref{alg:fgcp_update_s}.

\begin{algorithm}[H]
\caption{Fiberwise sparse-group update for the FG-Lasso anomaly tensor}
\label{alg:fgcp_update_s}
\begin{algorithmic}[1]
\Require Residual tensor
\(\mathcal R=\mathcal X-\mathcal L\),
anomaly regularization parameter \(\alpha_1\), and mixing parameter
\(\gamma\in[0,1]\)

\State
\(\displaystyle
\tau_{\mathrm e}
\gets
\frac{\alpha_1\gamma}{2}
\)

\State
\(\displaystyle
\tau_{\mathrm g}
\gets
\frac{\alpha_1(1-\gamma)}{2}
\)

\For{\(i=1,\ldots,I\)}
    \For{\(j=1,\ldots,J\)}
        \State
        \(\mathbf z\gets\mathcal R_{ij:}\)

        \State
        \(\mathbf u
        \gets
        \operatorname{soft}_{\tau_{\mathrm e}}(\mathbf z)\)

        \If{\(\|\mathbf u\|_2\leq\tau_{\mathrm g}\)}
            \State
            \(\mathcal S_{ij:}\gets\mathbf 0\)
        \Else
            \State
            \(\displaystyle
            \mathcal S_{ij:}
            \gets
            \left(
                1-
                \frac{\tau_{\mathrm g}}{\|\mathbf u\|_2}
            \right)
            \mathbf u
            \)
        \EndIf
    \EndFor
\EndFor

\State \Return \(\mathcal S\)
\end{algorithmic}
\end{algorithm}

For fixed \(\mathcal S\), define the anomaly-adjusted tensor
\begin{equation}
    \mathcal Y
    =
    \mathcal X-\mathcal S.
    \label{eq:fgcp_adjusted_tensor}
\end{equation}
The low-rank component is updated by fitting a regularized rank-\(R\) CP decomposition to \(\mathcal Y\). The parameter \(\alpha_2\) controls the ridge regularization applied to the temporal factor matrix \(\mathbf A^{(3)}\), whereas \(\alpha_1\) and \(\gamma\) affect only the
structured sparse update. The parameters \(\alpha_1\), \(\alpha_2\), and \(\gamma\) are selected from prespecified candidate sets. We use the same regularized CP-ALS framework described in Section~\ref{subsec:escp_optimization} and summarized in Algorithm~\ref{alg:escp_cpals}. Specifically, ridge regularization is applied to the temporal factor matrix, while the sample- and sensor-mode factors are updated through alternating least-squares steps. After the prescribed number of inner iterations, the low-rank component is reconstructed as $\mathcal L= \llbracket \mathbf A^{(1)}, \mathbf A^{(2)}, \mathbf A^{(3)} \rrbracket$. Thus, the distinction between ES-CP and FG-Lasso lies in the structural penalty and corresponding update for the anomaly tensor, rather than in the representation of the nominal low-rank component.

The complete FG-Lasso procedure initializes \(\mathcal S^{(0)}=\mathbf 0\) and alternates between the low-rank and structured sparse updates. At outer iteration \(t\), Algorithm~\ref{alg:escp_cpals} is applied to
\[
    \mathcal Y^{(t)}
    =
    \mathcal X-\mathcal S^{(t-1)}
\]
to obtain \(\mathcal L^{(t)}\). The residual
\[
    \mathcal R^{(t)}
    =
    \mathcal X-\mathcal L^{(t)}
\]
is then passed to Algorithm~\ref{alg:fgcp_update_s} to obtain \(\mathcal S^{(t)}\). These updates are repeated for a prespecified number of outer iterations. The two updates are repeated for a prescribed number of outer iterations. The implementation uses fixed inner and outer iteration counts rather than an objective-based convergence criterion. Because the CP model is nonconvex and the factor matrices are randomly initialized at each outer iteration,
the resulting estimates may depend on initialization. Therefore, monotonic decrease of \eqref{eq:fgcp_objective} and convergence to a global minimizer are not asserted.

\subsection{Sample-Level Anomaly Detection and Localization}
\label{subsec:anomaly_decision}

The proposed decomposition methods estimate an anomaly tensor \(\widehat{\mathcal S}\in\mathbb R^{I\times J\times K}\), where \(\widehat S_{ijk}\) represents the estimated anomalous deviation associated with sample \(i\), functional channel \(j\), and time point \(k\). To convert this tensor-valued estimate into a sample-level anomaly decision, we define the anomaly score for sample \(i\) as
\begin{equation}
    d_i
    =
    \left\|
        \widehat{\mathcal S}_{i::}
    \right\|_{\infty}
    =
    \max_{1\leq j\leq J}
    \max_{1\leq k\leq K}
    \left|
        \widehat S_{ijk}
    \right|,
    \qquad
    i=1,\ldots,I.
    \label{eq:sample_anomaly_score}
\end{equation}
Thus, \(d_i\) measures the largest estimated anomalous deviation within the multichannel functional observation corresponding to sample \(i\).

The sample-level anomaly indicator is defined as
\begin{equation}
    \widehat y_i
    =
    \mathbb I
    \left(
        d_i>\varepsilon_s
    \right),
    \label{eq:sample_anomaly_decision}
\end{equation}
where \(\mathbb I(\cdot)\) is the indicator function and \(\varepsilon_s\geq0\) is a support threshold. In the numerical studies, \(\varepsilon_s=0\) is used because the soft-thresholding and fiberwise proximal updates produce exact zeros. Consequently,
\begin{equation}
    \widehat y_i
    =
    \begin{cases}
        1,
        &
        \widehat S_{ijk}\neq0
        \text{ for at least one }(j,k),
        \\[4pt]
        0,
        &
        \widehat S_{ijk}=0
        \text{ for all }(j,k).
    \end{cases}
    \label{eq:sample_anomaly_decision_equivalent}
\end{equation}
A sample is therefore classified as faulty whenever its estimated sparse component contains at least one nonzero entry; otherwise, it is classified as normal.

The same sample-level decision rule is applied to ES-CP, FG-Lasso, and TRPCA so that the tensor-decomposition methods are compared using a common classification criterion. For the PCA-based benchmarks, the binary anomaly labels are obtained directly from Isolation Forest, Local Outlier Factor, and One-Class SVM.

In addition to sample-level detection, the support of \(\widehat{\mathcal S}\) provides diagnostic information. For a sample classified as faulty, the channel-level indicator
\begin{equation}
    \widehat y_{ij}
    =
    \mathbb I
    \left(
        \left\|
            \widehat{\mathcal S}_{ij:}
        \right\|_{\infty}
        >
        \varepsilon_s
    \right)
    \label{eq:channel_anomaly_indicator}
\end{equation}
identifies the functional channels containing estimated anomalous deviations, while the nonzero entries of \(\widehat{\mathcal S}_{ij:}\) identify their corresponding time locations. This hierarchical interpretation permits anomaly detection at the sample level and localization at the channel and time-point levels.

\section{Simulation}
\label{sec:Simulation}

We conduct two simulation studies to evaluate the anomaly-detection performance of the proposed ES-CP and FG-Lasso methods under different fault structures and data-quality conditions. The proposed methods are compared with Tensor Robust Principal Component Analysis (TRPCA), which is the most directly related benchmark because it also decomposes an observed tensor into low-rank and sparse components. The sparse component estimated by TRPCA can therefore be used for anomaly detection in a manner comparable to the proposed methods.

To provide a broader comparison with conventional dimension reduction based approaches, we also consider three PCA-based anomaly-detection methods: PCA combined with Isolation Forest (PCA+IF), Local Outlier Factor (PCA+LOF), and One-Class SVM (PCA+OCSVM). For each method, PCA is first used to project the functional observations onto the leading principal-component scores that explain at least \(95\%\) of the total variation. The resulting score vectors are then supplied to the corresponding unsupervised anomaly detector. Isolation Forest and LOF require a prespecified contamination proportion, while One-Class SVM requires selection of the parameter \(\nu\), which controls the allowable fraction of anomalous observations. To provide these benchmark methods with favorable operating conditions, these parameters are set using the known contamination proportion in each simulated dataset and the known fault proportion in the case study. In practice, however, the true anomaly proportion is generally unknown; therefore, the reported performance of these methods represents an oracle-assisted comparison.

The first simulation follows the multivariate functional-data generation framework considered by \cite{capezza2024robust}. Each observation contains multiple correlated functional channels generated over a common temporal domain. Controlled anomalies are introduced through casewise and componentwise contamination, with faults affecting either a single channel or all channels. The contamination magnitude is varied across several severity levels. This design provides a controlled setting for evaluating how the competing methods respond to localized, channel-specific, and broadly distributed anomalous deviations. 

The second simulation is inspired by the functional data generation setting of \cite{gahrooei2021multiple}. Each observation consists of multiple correlated functional channels generated over a common discretization domain. Complete sample--channel trajectories are masked for selected observations, and additive symmetric perturbations are introduced to represent data corruption. Rank-one persistent mean-shift anomalies of increasing severity are then added to a subset of the samples. This design evaluates the ability of the competing methods to distinguish nominal and anomalous variation under cross-channel dependence, trajectory-level masking, additive perturbations, and a limited sample size.

\subsection{Simulation I: Synthetic Multichannel Functional Process Simulation}
\label{subsec:simulation1}

This simulation study is to evaluate the performance of our method under controlled multichannel functional data. The data generation procedure follows established techniques from \cite{capezza2024robust, chiou2014multivariate, centofanti2021functional}. We define the functional domain as \( t \in [0, 1] \), and generate \( p = 10 \) functional signals for each sample. The generated multivariate functional data \( X(t) \) includes both normal and anomalous samples. Each sample comprises 10 channels, each representing one component of \( X(t) \), and the normal profiles are constructed using a specified mean function and a carefully designed covariance structure to reflect realistic behavior of functional signals.

Anomalies are introduced under either componentwise or casewise contamination. In the componentwise setting, contamination is generated at the sample--channel level, allowing different functional channels within the same sample to be contaminated independently. In the casewise setting, contamination is generated at the sample level, so that the channels specified by the scenario are contaminated jointly for a selected sample. For each contamination setting, we consider two channel patterns: the anomalous deviation is applied either to one selected functional channel or to all ten channels. The magnitude of the anomalous deviation is varied across four severity levels. Each simulation includes 1000 normal samples, with an additional 5\% of the samples contaminated by either type of anomaly. Figure~\ref{fig:Simulation data} illustrates an example of such contamination, where all signals exhibit a casewise anomaly and signal 5 is additionally affected by a componentwise outlier.


\begin{figure*}[ht]
\centering\includegraphics[width=1\linewidth]{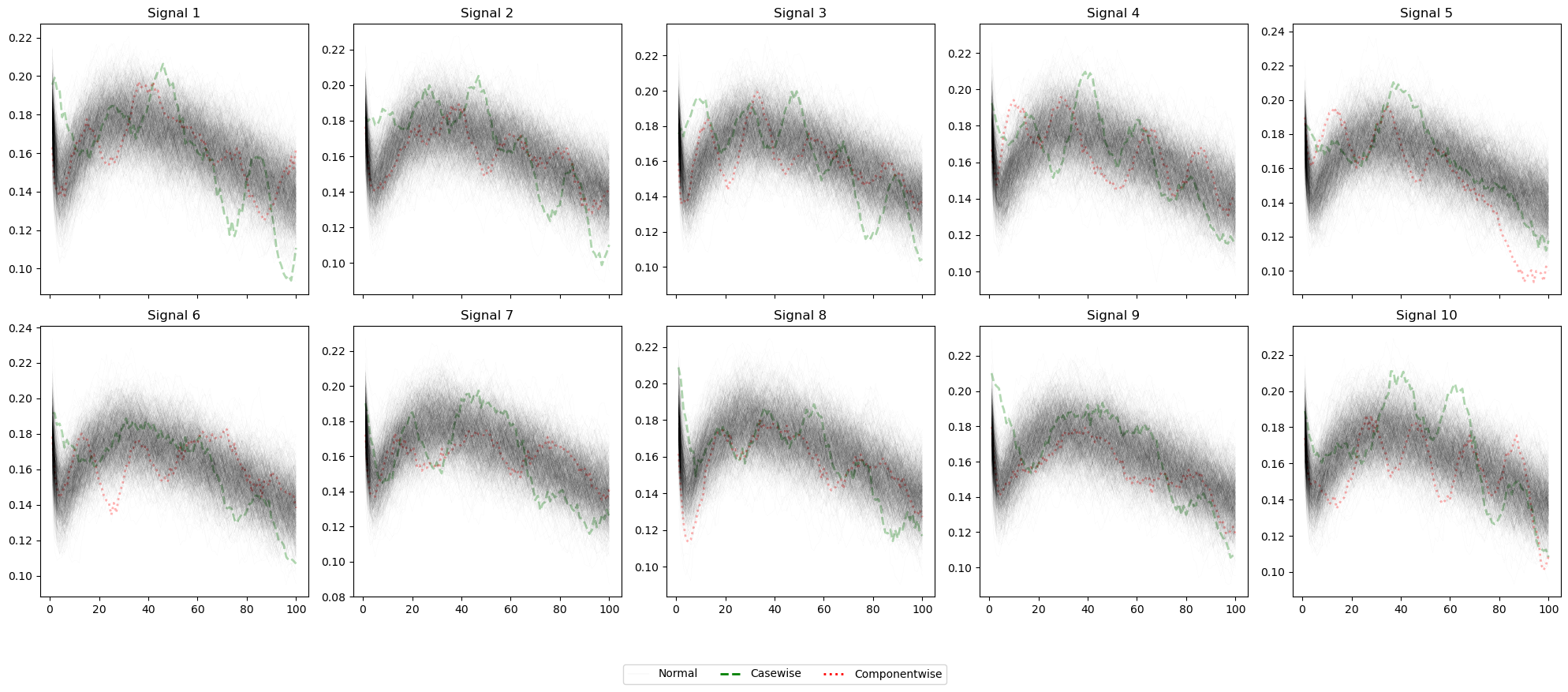}
\caption{Example of simulated multichannel functional data containing both casewise and componentwise contamination. All ten channels of the selected case contain casewise contamination, while Channel~5 additionally contains componentwise contamination.}
\label{fig:Simulation data}
\end{figure*}
 

To assess robustness, we define four levels of contamination severity following the same parameter settings as in \cite{capezza2024robust}, ranging from mild to extremely severe. Mild contamination represents subtle deviations that may be difficult to detect, while moderate contamination includes clearer, more visible deviations. Severe contamination results in large mean shifts that significantly depart from normal behavior, and extreme contamination introduces pronounced, high-magnitude shifts across the signals. This range allows us to comprehensively evaluate the sensitivity and stability of the proposed method under increasing anomaly intensity. Details of this simulation is given in  Appendix~\ref{app:SimulationIDetails}.


Tables~\ref{tab:ocp_all}--\ref{tab:oce_one} summarize the simulation results for the four combinations of contamination mechanism and fault pattern:
componentwise or casewise contamination affecting either a single functional channel or all ten channels. For each scenario, we report macro F$_1$, macro precision, and macro recall for FG-Lasso, ES-CP, PCA with Isolation Forest (PCA+IF), PCA with Local Outlier Factor (PCA+LOF), PCA with One-Class SVM (PCA+OCSVM), and TRPCA. Boldface denotes the highest macro F$_1$ score at
each severity level, with ties highlighted jointly.

Across the 16 scenario--severity combinations, FG-Lasso achieves or ties for
the highest macro F$_1$ score in 15 cases. Averaged across all settings, its
macro F$_1$ is 0.799, compared with 0.741 for ES-CP, 0.706 for TRPCA, and
0.554--0.587 for the three PCA-based anomaly detectors. The only setting in
which FG-Lasso does not attain the highest macro F$_1$ is the mildest
all-channel casewise case, where PCA+LOF obtains 0.559 compared with 0.543 for
both FG-Lasso and ES-CP. From severity level 2 onward, however, FG-Lasso
becomes the best-performing method in the casewise all-channel setting and its
advantage increases rapidly with contamination magnitude.

The benefit of the fiberwise penalty is most evident when the fault affects all channels. Under componentwise contamination, the macro F$_1$ of FG-Lasso increases from 0.617 at severity level 1 to 0.933 at severity level 4 (Table~\ref{tab:ocp_all}). At the highest severity, this exceeds ES-CP and TRPCA by 0.133 and 0.107, respectively. A similar pattern appears for the all-channel casewise condition (Table~\ref{tab:oce_all}), where FG-Lasso increases
from 0.543 to 0.956, compared with final macro F$_1$ scores of 0.792 for ES-CP and 0.824 for TRPCA. The corresponding precision and recall values of 0.953 and 0.958 show that the improvement is not driven by only one aspect of classification performance.

When contamination is restricted to one channel, FG-Lasso and ES-CP perform similarly. For the componentwise one-channel scenario
(Table~\ref{tab:ocp_one}), they tie at severity levels 2 and 4, with macro F$_1$ scores of 0.908 and 0.992, respectively. At the other two severity levels, FG-Lasso has only modest advantages of 0.017 and 0.006. The same behavior is observed for one-channel casewise contamination (Table~\ref{tab:oce_one}): FG-Lasso and ES-CP increase from 0.538 and 0.532 at severity level 1 to an identical value of 0.969 at severity level 4. Thus, when the anomaly is confined to a single trajectory, entrywise sparsity is generally sufficient, and the additional fiberwise penalty provides only a small improvement.

TRPCA is the strongest benchmark overall and improves substantially as the fault magnitude increases. It is particularly competitive in the all-channel scenarios, where it sometimes outperforms ES-CP. For example, at severity level 4, TRPCA obtains macro F$_1$ scores of 0.826 and 0.824 under componentwise and casewise all-channel contamination, compared with 0.800 and 0.792 for ES-CP. In contrast, ES-CP consistently exceeds TRPCA in the one-channel settings, reaching 0.992 and 0.969 under the highest-severity componentwise and casewise conditions. These results indicate that the entrywise CP formulation is particularly effective for localized channel-specific faults, whereas broader faults may be more readily captured by TRPCA unless fiber structure is explicitly incorporated.

The PCA-based detectors provide substantially weaker and less consistent performance. Their macro F$_1$ scores remain near 0.5--0.6 in many settings and often change little as severity increases. For example, in the componentwise all-channel scenario, PCA+IF decreases from 0.554 to 0.528 and PCA+LOF decreases from 0.601 to 0.501 as severity increases from 1 to 4. PCA+OCSVM is generally the strongest of the PCA-based methods in the all-channel settings, but its highest macro F$_1$ remains 0.670. This behavior suggests that reducing the functional tensor to conventional PCA scores can fail to preserve the sample--channel--time structure needed to distinguish systematic faults from nominal functional variation.

\begin{table*}[!htbp]
\centering
\setlength{\tabcolsep}{2.5pt}
\resizebox{\textwidth}{!}{%
\begin{tabular}{|c|ccc|ccc|ccc|ccc|ccc|ccc|}
\hline
\multirow{2}{*}{Severity}
& \multicolumn{3}{c|}{FG-Lasso}
& \multicolumn{3}{c|}{ES-CP}
& \multicolumn{3}{c|}{PCA+IF}
& \multicolumn{3}{c|}{PCA+LOF}
& \multicolumn{3}{c|}{PCA+OCSVM}
& \multicolumn{3}{c|}{TRPCA} \\ \cline{2-19}
& Macro F$_1$ & Macro Precision & Macro Recall
& Macro F$_1$ & Macro Precision & Macro Recall
& Macro F$_1$ & Macro Precision & Macro Recall
& Macro F$_1$ & Macro Precision & Macro Recall
& Macro F$_1$ & Macro Precision & Macro Recall
& Macro F$_1$ & Macro Precision & Macro Recall \\ \hline

1
& \textbf{0.617} & 0.612 & 0.624
& 0.561 & 0.588 & 0.549
& 0.554 & 0.554 & 0.554
& 0.601 & 0.601 & 0.601
& 0.584 & 0.580 & 0.588
& 0.556 & 0.582 & 0.544 \\ \hline

2
& \textbf{0.716} & 0.789 & 0.674
& 0.581 & 0.784 & 0.553
& 0.554 & 0.554 & 0.554
& 0.606 & 0.606 & 0.606
& 0.605 & 0.601 & 0.609
& 0.644 & 0.675 & 0.624 \\ \hline

3
& \textbf{0.857} & 0.883 & 0.835
& 0.715 & 0.891 & 0.653
& 0.533 & 0.533 & 0.533
& 0.548 & 0.548 & 0.548
& 0.613 & 0.610 & 0.615
& 0.759 & 0.764 & 0.754 \\ \hline

4
& \textbf{0.933} & 0.940 & 0.927
& 0.800 & 0.947 & 0.729
& 0.528 & 0.527 & 0.527
& 0.501 & 0.501 & 0.501
& 0.652 & 0.645 & 0.661
& 0.826 & 0.806 & 0.849 \\ \hline

\end{tabular}}
\caption{\footnotesize
Macro F$_1$, macro precision, and macro recall under componentwise contamination affecting all ten functional channels.}
\label{tab:ocp_all}
\end{table*}

\begin{table*}[!htbp]
\centering
\setlength{\tabcolsep}{2.5pt}
\resizebox{\textwidth}{!}{%
\begin{tabular}{|c|ccc|ccc|ccc|ccc|ccc|ccc|}
\hline
\multirow{2}{*}{Severity}
& \multicolumn{3}{c|}{FG-Lasso}
& \multicolumn{3}{c|}{ES-CP}
& \multicolumn{3}{c|}{PCA+IF}
& \multicolumn{3}{c|}{PCA+LOF}
& \multicolumn{3}{c|}{PCA+OCSVM}
& \multicolumn{3}{c|}{TRPCA} \\ \cline{2-19}
& Macro F$_1$ & Macro Precision & Macro Recall
& Macro F$_1$ & Macro Precision & Macro Recall
& Macro F$_1$ & Macro Precision & Macro Recall
& Macro F$_1$ & Macro Precision & Macro Recall
& Macro F$_1$ & Macro Precision & Macro Recall
& Macro F$_1$ & Macro Precision & Macro Recall \\ \hline

1
& 0.543 & 0.537 & 0.561
& 0.543 & 0.537 & 0.561
& 0.522 & 0.522 & 0.522
& \textbf{0.559} & 0.559 & 0.559
& 0.547 & 0.544 & 0.550
& 0.504 & 0.510 & 0.504 \\ \hline

2
& \textbf{0.665} & 0.627 & 0.782
& 0.542 & 0.612 & 0.530
& 0.522 & 0.522 & 0.522
& 0.585 & 0.585 & 0.585
& 0.626 & 0.624 & 0.627
& 0.573 & 0.602 & 0.559 \\ \hline

3
& \textbf{0.826} & 0.799 & 0.858
& 0.688 & 0.798 & 0.641
& 0.564 & 0.564 & 0.564
& 0.491 & 0.491 & 0.491
& 0.623 & 0.618 & 0.630
& 0.658 & 0.685 & 0.639 \\ \hline

4
& \textbf{0.956} & 0.953 & 0.958
& 0.792 & 0.954 & 0.719
& 0.559 & 0.559 & 0.559
& 0.480 & 0.480 & 0.480
& 0.670 & 0.657 & 0.685
& 0.824 & 0.803 & 0.849 \\ \hline

\end{tabular}}
\caption{\footnotesize
Macro F$_1$, macro precision, and macro recall under casewise contamination affecting all ten functional channels.}
\label{tab:oce_all}
\end{table*}

\begin{table*}[!htbp]
\centering
\setlength{\tabcolsep}{2.5pt}
\resizebox{\textwidth}{!}{%
\begin{tabular}{|c|ccc|ccc|ccc|ccc|ccc|ccc|}
\hline
\multirow{2}{*}{Severity}
& \multicolumn{3}{c|}{FG-Lasso}
& \multicolumn{3}{c|}{ES-CP}
& \multicolumn{3}{c|}{PCA+IF}
& \multicolumn{3}{c|}{PCA+LOF}
& \multicolumn{3}{c|}{PCA+OCSVM}
& \multicolumn{3}{c|}{TRPCA} \\ \cline{2-19}
& Macro F$_1$ & Macro Precision & Macro Recall
& Macro F$_1$ & Macro Precision & Macro Recall
& Macro F$_1$ & Macro Precision & Macro Recall
& Macro F$_1$ & Macro Precision & Macro Recall
& Macro F$_1$ & Macro Precision & Macro Recall
& Macro F$_1$ & Macro Precision & Macro Recall \\ \hline

1
& \textbf{0.782} & 0.856 & 0.736
& 0.765 & 0.863 & 0.711
& 0.554 & 0.554 & 0.554
& 0.559 & 0.559 & 0.559
& 0.527 & 0.526 & 0.527
& 0.676 & 0.706 & 0.655 \\ \hline

2
& \textbf{0.908} & 0.910 & 0.906
& \textbf{0.908} & 0.910 & 0.906
& 0.580 & 0.580 & 0.580
& 0.617 & 0.617 & 0.617
& 0.567 & 0.564 & 0.571
& 0.816 & 0.804 & 0.830 \\ \hline

3
& \textbf{0.969} & 0.964 & 0.973
& 0.963 & 0.963 & 0.963
& 0.570 & 0.569 & 0.569
& 0.648 & 0.648 & 0.648
& 0.573 & 0.569 & 0.577
& 0.891 & 0.845 & 0.955 \\ \hline

4
& \textbf{0.992} & 0.985 & 0.999
& \textbf{0.992} & 0.985 & 0.999
& 0.575 & 0.575 & 0.575
& 0.638 & 0.638 & 0.638
& 0.584 & 0.580 & 0.588
& 0.905 & 0.851 & 0.984 \\ \hline

\end{tabular}}
\caption{\footnotesize
Macro F$_1$, macro precision, and macro recall under componentwise contamination
affecting one functional channel.}
\label{tab:ocp_one}
\end{table*}

\begin{table*}[!htbp]
\centering
\setlength{\tabcolsep}{2.5pt}
\resizebox{\textwidth}{!}{%
\begin{tabular}{|c|ccc|ccc|ccc|ccc|ccc|ccc|}
\hline
\multirow{2}{*}{Severity}
& \multicolumn{3}{c|}{FG-Lasso}
& \multicolumn{3}{c|}{ES-CP}
& \multicolumn{3}{c|}{PCA+IF}
& \multicolumn{3}{c|}{PCA+LOF}
& \multicolumn{3}{c|}{PCA+OCSVM}
& \multicolumn{3}{c|}{TRPCA} \\ \cline{2-19}
& Macro F$_1$ & Macro Precision & Macro Recall
& Macro F$_1$ & Macro Precision & Macro Recall
& Macro F$_1$ & Macro Precision & Macro Recall
& Macro F$_1$ & Macro Precision & Macro Recall
& Macro F$_1$ & Macro Precision & Macro Recall
& Macro F$_1$ & Macro Precision & Macro Recall \\ \hline

1
& \textbf{0.538} & 0.532 & 0.563
& 0.532 & 0.529 & 0.539
& 0.522 & 0.522 & 0.522
& 0.517 & 0.517 & 0.517
& 0.506 & 0.506 & 0.506
& 0.490 & 0.488 & 0.495 \\ \hline

2
& \textbf{0.646} & 0.645 & 0.648
& 0.635 & 0.660 & 0.618
& 0.564 & 0.564 & 0.564
& 0.580 & 0.580 & 0.580
& 0.547 & 0.546 & 0.547
& 0.537 & 0.559 & 0.530 \\ \hline

3
& \textbf{0.870} & 0.912 & 0.837
& 0.867 & 0.911 & 0.832
& 0.564 & 0.564 & 0.564
& 0.638 & 0.638 & 0.638
& 0.581 & 0.579 & 0.584
& 0.738 & 0.754 & 0.725 \\ \hline

4
& \textbf{0.969} & 0.961 & 0.978
& \textbf{0.969} & 0.961 & 0.978
& 0.596 & 0.596 & 0.596
& 0.627 & 0.627 & 0.627
& 0.580 & 0.574 & 0.586
& 0.899 & 0.849 & 0.970 \\ \hline

\end{tabular}}
\caption{\footnotesize
Macro F$_1$, macro precision, and macro recall under casewise contamination
affecting one functional channel.}
\label{tab:oce_one}
\end{table*}

\subsection{Simulation II: Multivariate Functional Data with Missingness and Anomalies} \label{subsec:simulation2} To evaluate anomaly-detection performance under structured data corruption, we simulate a multivariate functional dataset inspired by \cite{gahrooei2021multiple}. The simulation contains \(N=200\) samples, each consisting of \(p=5\) correlated functional channels observed at \(T=100\) discretization points. Complete temporal trajectories are masked for selected sample--channel pairs, after which additive symmetric perturbations are introduced. Persistent rank-one mean-shift anomalies are injected into \(5\%\) of the samples, and their magnitude is varied across the five severity levels reported in Table~\ref{tab:simulation2_comparison}. Additional details of the data-generation procedure are provided in Appendix~\ref{app:SimulationIIDetails}.

Table~\ref{tab:simulation2_comparison} reports results from this simulation with similar benchmarks to simulation I. As we can see in ~\ref{tab:simulation2_comparison} FG-Lasso provides the strongest overall performance and achieves the highest macro F$_1$ score at four of the five severity levels. Averaged across all severity settings, FG-Lasso obtains a macro F$_1$ of 0.711, compared with
0.681 for ES-CP, 0.645 for PCA+OCSVM, 0.642 for PCA+LOF, 0.621 for PCA+IF, and 0.609 for TRPCA. The only exception occurs at the smallest severity, 0.25, where PCA+LOF and PCA+IF attain the highest macro F$_1$ of 0.579. FG-Lasso and PCA+OCSVM follow closely with macro F$_1$ scores of 0.578, indicating that the anomaly signal is difficult to distinguish from nominal variation at this severity.

As the anomaly magnitude increases, the benefit of modeling the tensor and fiber structures becomes more pronounced. At severity 0.50, FG-Lasso obtains a macro F$_1$ of 0.601, compared with 0.579 for ES-CP and PCA+LOF and 0.519 for TRPCA. At severity 1.00, FG-Lasso increases to 0.673, exceeding PCA+OCSVM, ES-CP, and TRPCA, which obtain 0.623, 0.598, and 0.523, respectively. This setting produces the largest separation between FG-Lasso and ES-CP, suggesting that the fiberwise penalty is especially beneficial when anomalies are detectable but remain partially obscured by trajectory masking, additive perturbations, and background variation.

Both proposed tensor methods perform strongly at the two highest severity levels. At severity 2.00, FG-Lasso and ES-CP obtain macro F$_1$ scores of 0.785 and 0.778, respectively, substantially exceeding TRPCA at 0.663 and the three PCA-based detectors at 0.684. At severity 3.00, their macro F$_1$ scores increase further to 0.917 and 0.904, whereas TRPCA reaches 0.823. Thus, although the performance gap between the two proposed methods becomes small for pronounced anomalies, both retain a clear advantage over the benchmarking approaches.

\begin{table*}[!htbp]
\centering
\setlength{\tabcolsep}{2.5pt}
\resizebox{\textwidth}{!}{%
\begin{tabular}{|c|ccc|ccc|ccc|ccc|ccc|ccc|}
\hline
\multirow{2}{*}{Severity}
& \multicolumn{3}{c|}{FG-Lasso}
& \multicolumn{3}{c|}{ES-CP}
& \multicolumn{3}{c|}{PCA+IF}
& \multicolumn{3}{c|}{PCA+LOF}
& \multicolumn{3}{c|}{PCA+OCSVM}
& \multicolumn{3}{c|}{TRPCA} \\ \cline{2-19}

& Macro F$_1$ & Macro Precision & Macro Recall
& Macro F$_1$ & Macro Precision & Macro Recall
& Macro F$_1$ & Macro Precision & Macro Recall
& Macro F$_1$ & Macro Precision & Macro Recall
& Macro F$_1$ & Macro Precision & Macro Recall
& Macro F$_1$ & Macro Precision & Macro Recall \\ \hline

0.25
& \textbf{0.58} & 0.56 & 0.61
& 0.55 & 0.54 & 0.56
& \textbf{0.58} & 0.58 & 0.58
& \textbf{0.58} & 0.58 & 0.58
& \textbf{0.58} & 0.56 & 0.61
& 0.52 & 0.52 & 0.52 \\ \hline

0.50
& \textbf{0.60} & 0.59 & 0.62
& 0.58 & 0.58 & 0.58
& 0.53 & 0.53 & 0.53
& 0.58 & 0.58 & 0.58
& 0.56 & 0.55 & 0.60
& 0.52 & 0.52 & 0.52 \\ \hline

1.00
& \textbf{0.67} & 0.67 & 0.68
& 0.60 & 0.58 & 0.69
& 0.58 & 0.58 & 0.58
& 0.58 & 0.58 & 0.58
& 0.62 & 0.60 & 0.67
& 0.52 & 0.52 & 0.52 \\ \hline

2.00
& \textbf{0.79} & 0.84 & 0.75
& 0.78 & 0.74 & 0.83
& 0.68 & 0.68 & 0.68
& 0.68 & 0.68 & 0.68
& 0.68 & 0.68 & 0.68
& 0.66 & 0.65 & 0.68 \\ \hline

3.00
& \textbf{0.92} & 0.94 & 0.90
& 0.90 & 0.87 & 0.94
& 0.74 & 0.74 & 0.74
& 0.79 & 0.79 & 0.79
& 0.78 & 0.74 & 0.83
& 0.82 & 0.78 & 0.88 \\ \hline

\end{tabular}}
\caption{\footnotesize
Macro F$_1$, macro precision, and macro recall for the six anomaly-detection methods in Simulation II across increasing anomaly severity levels. The simulation includes multivariate functional observations with missing entries, measurement noise, and anomalous deviations. 
}
\label{tab:simulation2_comparison}
\end{table*}

\section{Case Study}\label{sec:Case}

In this section, the proposed methodology is applied to monitor multichannel tonnage profile data collected from a multi-operation forging process. The forging machine used in this process consists of multiple dies, each performing a specific operation during a stroke. The tonnage proﬁle dataset in Figure \ref{fig:tonnageintro} was used in previous studies by \cite{paynabar2016change,paynabar2013analysis, lei2010automatic}. 

To monitor the process performance, four strain gauge sensors (C1, C2, C3, and C4) are mounted on the columns of the forging machine. These sensors measure the tonnage force exerted by the press uprights, resulting in a four-channel tonnage profile for each operation cycle. Each profile channel is 1200 data points long. A sample of 651 multichannel profiles was collected under various experimental conditions. This sample includes 306 in-control profiles (collected under normal production conditions) and five groups of 69 out-of-control profiles. Each out-of-control group corresponds to a specific faulty operation condition where a part is missing in one of the operation stations, such as piercing in Station 4 (referred to as the C1F4, C2F4, C3F4,and C4F4 group). The faulty conditions can affect individual profile channels differently, depending on the station's proximity to the sensors measuring the profiles.

Unlike the supervised approaches of Lei et al.~\cite{lei2010automatic} and Paynabar et al.~\cite{paynabar2013analysis}, which used subsets of labeled profiles for model training and testing, the present case study evaluates all methods in a fully unsupervised setting. It also differs from the Phase~I and Phase~II control-chart framework of Paynabar et al.~\cite{paynabar2016change}; rather than estimating control limits from a reference sample, the objective here is to identify anomalous profiles directly from the complete dataset.

Some fault conditions produce deviations that are visually evident without specialized modeling and were therefore excluded from the analysis. The resulting dataset contains 513 multichannel profiles: 306 normal profiles and 207 faulty profiles from three retained fault groups, with 69 profiles in each group. The profiles were arranged as a third-order tensor whose modes represent samples, sensor channels, and crank-angle positions. Each raw profile channel contained 1,201 crank-angle observations. Before model fitting, the first 1,200 observations were partitioned into 200 consecutive, nonoverlapping blocks of six observations, and each block was replaced by its mean; the final unmatched observation was discarded. This preprocessing produced the analysis tensor \(\mathcal X\in\mathbb R^{513\times200\times4}\), whose modes correspond to strokes, block-averaged crank-angle positions, and strain-gauge channels, respectively. FG-Lasso, ES-CP, and TRPCA were applied directly to this tensor. For comparison, three additional unsupervised baselines were considered by first extracting principal-component scores and then applying Isolation Forest, Local Outlier Factor, or One-Class SVM.

The quantitative results in Table~\ref{tab:model_comparison} show that the two proposed tensor-decomposition methods outperform both TRPCA and the three PCA-based anomaly detectors. Unlike the simulation studies, in which performance was evaluated separately
across multiple contamination mechanisms, fault patterns, and severity levels, the case study was formulated as a binary anomaly-detection problem. All observations associated with the different fault types were pooled into a single Faulty class, while the remaining observations were assigned to the Normal class. We therefore report class-specific precision, recall, and F$_1$ scores for both classes rather than relying only on macro-averaged
performance measures.
FG-Lasso provides the strongest overall performance, with an F$_1$ score of 0.89 for normal profiles and 0.80 for faulty profiles. It correctly classifies 299 of the 306 normal profiles, corresponding to a normal-class recall of 0.98 and a false-alarm rate of only 0.02. Its faulty-class precision is 0.95, indicating that nearly all profiles identified as anomalous are truly faulty, although its faulty-class recall of 0.69 shows that some subtle faults remain undetected.

ES-CP is the second-best method, attaining F$_1$ scores of 0.86 and 0.78 for the normal and faulty classes, respectively. Compared with FG-Lasso, ES-CP has lower faulty-class precision, 0.83 versus 0.95, but slightly higher faulty-class recall, 0.73 versus 0.69. Thus, ES-CP detects more of the faulty profiles, whereas FG-Lasso produces substantially fewer false alarms and false positive anomaly declarations. Their macro F$_1$ scores are 0.82 and 0.85, respectively, confirming a modest but consistent overall advantage for the structured fiberwise penalty.

Among the PCA-based methods, PCA+OCSVM performs best, with F$_1$ scores of 0.75 for normal profiles and 0.63 for faulty profiles. PCA+IF obtains corresponding values of 0.73 and 0.59, while PCA+LOF performs less favorably, with values of 0.63 and 0.45. These methods therefore improve only modestly, or not at all, upon TRPCA, whose normal- and faulty-class F$_1$ scores are 0.65 and 0.47. In particular, TRPCA and PCA+LOF detect only 94 of the 207 faulty profiles, corresponding to a faulty-class recall of 0.45. The weaker performance of the PCA-based detectors suggests that projecting the profiles onto an unstructured score space discards part of the localized sensor--time information needed to distinguish faulty operation cycles from normal process variability.

\begin{table}[!htbp]
    \centering
    \resizebox{0.78\textwidth}{!}{%
    \begin{tabular}{|l|l|c|c|c|c|c|c|}
        \hline
        \textbf{Class}
        & \textbf{Metric}
        & \textbf{FG-Lasso}
        & \textbf{ES-CP}
        & \textbf{PCA+IF}
        & \textbf{PCA+LOF}
        & \textbf{PCA+OCSVM}
        & \textbf{TRPCA} \\ \hline

        \multirow{3}{*}{Normal}
        & Precision
        & 0.82 & 0.83 & 0.73 & 0.63 & 0.75 & 0.64 \\

        & Recall
        & 0.98 & 0.90 & 0.73 & 0.63 & 0.75 & 0.66 \\

        & F$_1$ score
        & \textbf{0.89} & 0.86 & 0.73 & 0.63 & 0.75 & 0.65 \\ \hline

        \multirow{3}{*}{Faulty}
        & Precision
        & 0.95 & 0.83 & 0.59 & 0.45 & 0.63 & 0.48 \\

        & Recall
        & 0.69 & 0.73 & 0.59 & 0.45 & 0.63 & 0.45 \\

        & F$_1$ score
        & \textbf{0.80} & 0.78 & 0.59 & 0.45 & 0.63 & 0.47 \\ \hline
    \end{tabular}}
    \caption{\footnotesize
    Class-specific precision, recall, and F$_1$ scores for the six
    unsupervised anomaly-detection methods applied to the forging-process tonnage profiles.
    }
    \label{tab:model_comparison}
\end{table}

The low-rank reconstructions in Figure \ref{fig:Case_LowRank} reveal that FG-Lasso retains only the smooth, process driven tonnage patterns, while TRPCA still carries noticeable fault signatures, particularly around peak force instants, hinting at rank misallocation. Likewise, the sparse tensors in Figure \ref{fig:Case_Sparse} illustrate that FG-Lasso localises the faults into crisp, high-contrast patches confined to the affected channels, whereas TRPCA produces a more diffused anomaly map, masking subtle deviations. This confirms that enforcing group-sparsity along the time fibres helps separate transient shocks and operation-specific disturbances from the baseline press dynamics.

\begin{figure}[h]
    \centering
    \includegraphics[width=1\linewidth]{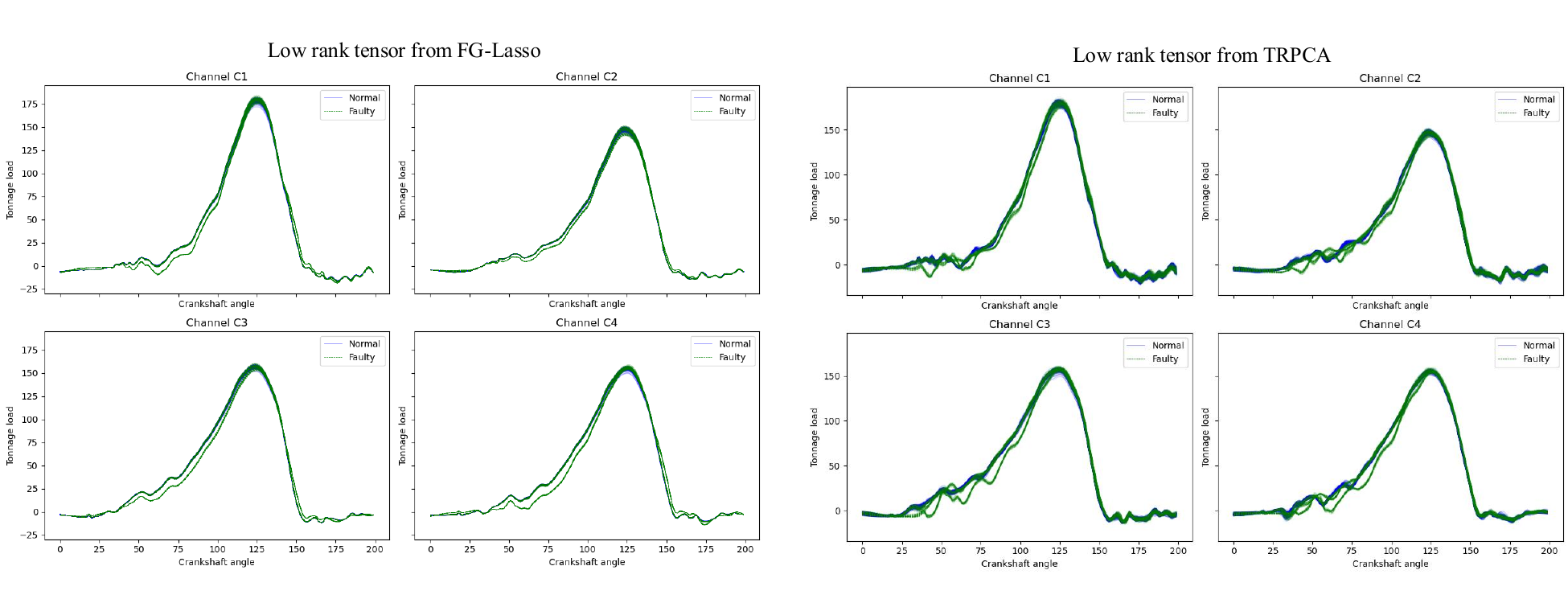}
    \caption{\scriptsize
        Low rank reconstruction of the four-sensor tonnage profiles obtained with FG-Lasso (left) and TRPCA (right).  
        }
    \label{fig:Case_LowRank}
\end{figure}

\begin{figure}[h]
    \centering
    \includegraphics[width=1\linewidth]{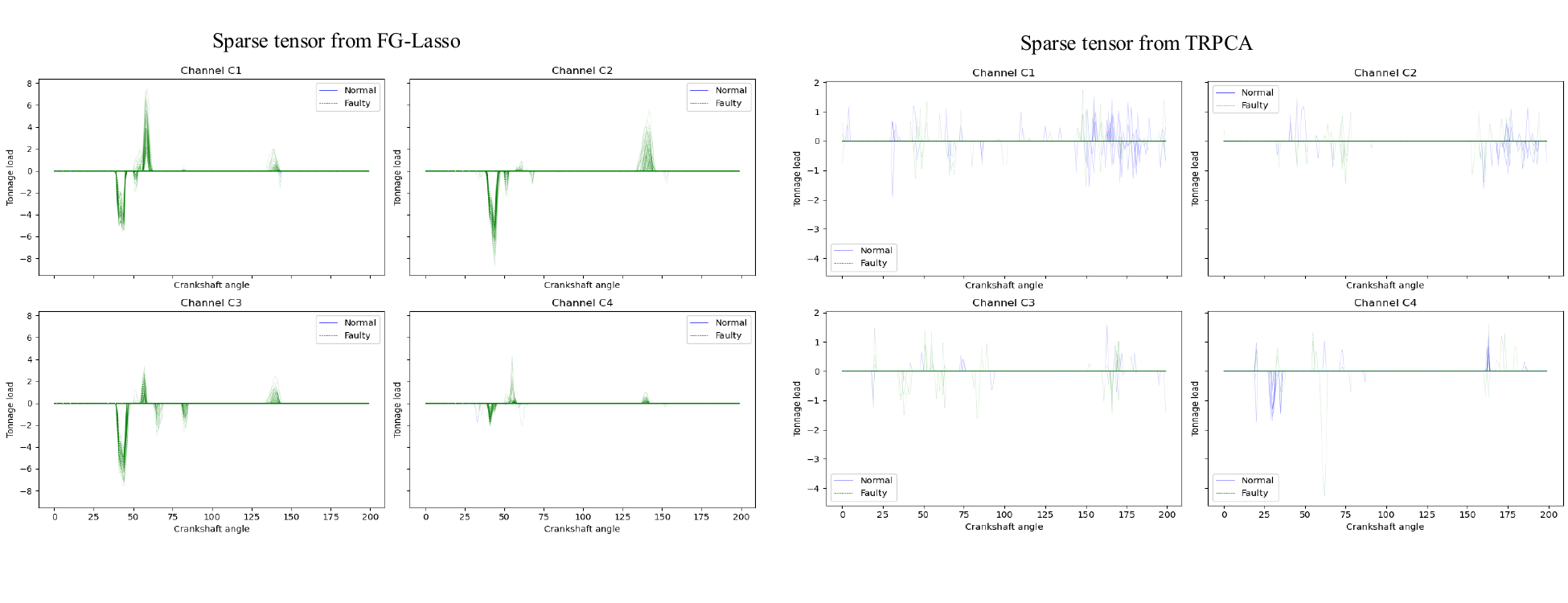}
    \caption{\scriptsize
        Sparse‐component reconstructions produced by FG–Lasso (left column) and TRPCA (right column) for the four tonnage-strain gauges C1 to C4. Each plot includes 513 strokes ($y$–axis) by 1200 crank-angle positions ($x$–axis);%
        }
    \label{fig:Case_Sparse}
\end{figure}

\section{Conclusion}\label{sec:Conclusion}

This paper develops two robust tensor decomposition methods for unsupervised anomaly detection in multivariate functional data. Both methods decompose the observed tensor into a low-rank component representing nominal process behavior and a sparse component capturing anomalous deviations. ES-CP imposes entrywise sparsity on the anomaly tensor and is therefore well suited to localized deviations, whereas FG-Lasso combines entrywise and fiberwise penalties to capture anomalies that persist along temporal trajectories or affect multiple sensor channels. Ridge regularization of the temporal factor is additionally used to stabilize estimation of the low-rank component.

The results show that the appropriate sparsity structure depends on the anticipated fault mechanism. ES-CP provides a competitive and relatively simple formulation for isolated or channel-specific anomalies. FG-Lasso retains this sensitivity while offering improved precision and stronger performance when anomalous deviations are organized within temporal fibers or distributed across multiple channels. The proposed framework therefore provides both effective anomaly detection and interpretable fault localization for high-dimensional industrial processes characterized by correlated, noisy, and incomplete functional measurements.

Future research may extend the proposed methods to online and streaming monitoring, where the low-rank and sparse components must be updated as new functional observations arrive. Other important directions include automatic selection of the CP rank and regularization parameters, uncertainty quantification for anomaly decisions, direct treatment of irregularly sampled or partially observed profiles, and extensions to heterogeneous multimodal process data.

\bibliographystyle{plainnat}
\bibliography{references}

@article{chiou2014multivariate,
  title={Multivariate functional principal component analysis: A normalization approach},
  author={Chiou, Jeng-Min and Chen, Yu-Ting and Yang, Ya-Fang},
  journal={Statistica Sinica},
  pages={1571--1596},
  year={2014},
  publisher={JSTOR}
}

@article{capezza2024robust,
  title={Robust multivariate functional control chart},
  author={Capezza, Christian and Centofanti, Fabio and Lepore, Antonio and Palumbo, Biagio},
  journal={Technometrics},
  pages={1--17},
  year={2024},
  publisher={Taylor \& Francis}
}

@article{centofanti2021functional,
  title={Functional regression control chart},
  author={Centofanti, Fabio and Lepore, Antonio and Menafoglio, Alessandra and Palumbo, Biagio and Vantini, Simone},
  journal={Technometrics},
  volume={63},
  number={3},
  pages={281--294},
  year={2021},
  publisher={Taylor \& Francis}
}

@inproceedings{schwab2012improving,
  title={Improving expert knowledge in dynamic process monitoring by symbolic regression},
  author={Schwab, Ingo and Senn, Melanie and Link, Norbert},
  booktitle={2012 Sixth International Conference on Genetic and Evolutionary Computing},
  pages={132--135},
  year={2012},
  organization={IEEE}
}

@article{paynabar2016change,
  title={A change-point approach for phase-I analysis in multivariate profile monitoring and diagnosis},
  author={Paynabar, Kamran and Zou, Changliang and Qiu, Peihua},
  journal={Technometrics},
  volume={58},
  number={2},
  pages={191--204},
  year={2016},
  publisher={Taylor \& Francis}
}

@article{lei2010automatic,
  title={Automatic tonnage monitoring for missing part detection in multi-operation forging processes},
  author={Lei, Yong and Zhang, Zhisheng and Jin, Jionghua},
  year={2010}
}

@article{paynabar2013analysis,
  title={Analysis of multichannel nonlinear profiles using uncorrelated multilinear principal component analysis with applications in fault detection and diagnosis},
  author={Paynabar, K and Jin, J and Pacella, M},
  journal={IIE Transactions},
  volume={45},
  number={11},
  pages={1235--1247},
  year={2013}
}

@article{gahrooei2021multiple,
  title={Multiple tensor-on-tensor regression: An approach for modeling processes with heterogeneous sources of data},
  author={Gahrooei, Mostafa Reisi and Yan, Hao and Paynabar, Kamran and Shi, Jianjun},
  journal={Technometrics},
  volume={63},
  number={2},
  pages={147--159},
  year={2021},
  publisher={Taylor \& Francis}
}

@book{kokoszka2017introduction,
  title={Introduction to functional data analysis},
  author={Kokoszka, Piotr and Reimherr, Matthew},
  year={2017},
  publisher={Chapman and Hall/CRC}
}

@article{ramsay2005principal,
  title={Principal components analysis for functional data},
  author={Ramsay, James O and Silverman, Bernard W},
  journal={Functional data analysis},
  pages={147--172},
  year={2005},
  publisher={Springer}
}

@article{menafoglio2018profile,
  title={Profile monitoring of probability density functions via simplicial functional PCA with application to image data},
  author={Menafoglio, Alessandra and Grasso, Marco and Secchi, Piercesare and Colosimo, Bianca Maria},
  journal={Technometrics},
  volume={60},
  number={4},
  pages={497--510},
  year={2018},
  publisher={Taylor \& Francis}
}

@article{gomez2020functional,
  title={Functional directed graphical models and applications in root-cause analysis and diagnosis},
  author={G{\'o}mez, Ana Mar{\'\i}a Estrada and Paynabar, Kamran and Pacella, Massimo},
  journal={Journal of Quality Technology},
  volume={53},
  number={4},
  pages={421--437},
  year={2020},
  publisher={Taylor \& Francis}
}

@article{fang2021multi,
  title={Multi-sensor prognostics modeling for applications with highly incomplete signals},
  author={Fang, Xiaolei and Yan, Hao and Gebraeel, Nagi and Paynabar, Kamran},
  journal={IISE Transactions},
  volume={53},
  number={5},
  pages={597--613},
  year={2021},
  publisher={Taylor \& Francis}
}

@article{yan2018real,
  title={Real-time monitoring of high-dimensional functional data streams via spatio-temporal smooth sparse decomposition},
  author={Yan, Hao and Paynabar, Kamran and Shi, Jianjun},
  journal={Technometrics},
  volume={60},
  number={2},
  pages={181--197},
  year={2018},
  publisher={Taylor \& Francis}
}

@article{jalilibal2024control,
  title={Control charting methods for monitoring high dimensional data streams: A conceptual classification scheme},
  author={Jalilibal, Zahra and Karavigh, Mohammad Hassan Ahmadi and Maleki, Mohammad Reza and Amiri, Amirhossein},
  journal={Computers \& Industrial Engineering},
  volume={191},
  pages={110141},
  year={2024},
  publisher={Elsevier}
}

@article{mou2023compressed,
  title={Compressed smooth sparse decomposition},
  author={Mou, Shancong and Shi, Jianjun},
  journal={INFORMS Journal on Data Science},
  volume={2},
  number={1},
  pages={60--80},
  year={2023},
  publisher={INFORMS}
}

@article{han2024guaranteed,
  title={Guaranteed functional tensor singular value decomposition},
  author={Han, Rungang and Shi, Pixu and Zhang, Anru R},
  journal={Journal of the American Statistical Association},
  volume={119},
  number={546},
  pages={995--1007},
  year={2024},
  publisher={Taylor \& Francis}
}

@article{yang2025anomaly,
  title={Anomaly detection in PV systems using constrained low-rank and sparse decomposition},
  author={Yang, Wei and Fregosi, Daniel and Bolen, Michael and Paynabar, Kamran},
  journal={IISE Transactions},
  volume={57},
  number={6},
  pages={607--620},
  year={2025},
  publisher={Taylor \& Francis}
}

@article{li2025real,
  title={Real-Time Monitoring of Dynamic Tensor Data with Longitudinal Patterns: A Tensor Graphical LASSO Approach},
  author={Li, Wendong and Li, Yifan and Tsung, Fugee and Wu, Chunjie},
  journal={Technometrics},
  number={just-accepted},
  pages={1--23},
  year={2025},
  publisher={Taylor \& Francis}
}

@article{qiu2023spatio,
  title={Spatio-temporal process monitoring using exponentially weighted spatial LASSO},
  author={Qiu, Peihua and Yang, Kai},
  journal={Journal of Quality Technology},
  volume={55},
  number={2},
  pages={163--180},
  year={2023},
  publisher={Taylor \& Francis}
}

@article{nabhan2021correlation,
  title={Correlation-based dynamic sampling for online high dimensional process monitoring},
  author={Nabhan, Mohammad and Mei, Yajun and Shi, Jianjun},
  journal={Journal of Quality Technology},
  volume={53},
  number={3},
  pages={289--308},
  year={2021},
  publisher={Taylor \& Francis}
}

@article{perdikis2019survey,
  author    = {Perdikis, Theodoros and Psarakis, Stelios},
  title     = {A survey on multivariate adaptive control charts: Recent developments and extensions},
  journal   = {Quality and Reliability Engineering International},
  volume    = {35},
  number    = {5},
  pages     = {1342--1362},
  year      = {2019},
  publisher = {Wiley},
  doi       = {10.1002/qre.2476}
}

@article{oyegoke2025hotelling,
  title={Hotelling T2 control chart based on minimum vector variance for monitoring high-dimensional correlated multivariate process},
  author={Oyegoke, Odunayo Adiat and Adekeye, Kayode Samuel and Olaomi, John Olutunji and Malela-Majika, Jean-Claude},
  journal={Quality and Reliability Engineering International},
  volume={41},
  number={2},
  pages={765--783},
  year={2025},
  publisher={Wiley Online Library}
}

@article{sun2024ewma,
  title={An EWMA chart for high dimensional process with multi-class out-of-control information via random forest learning},
  author={Sun, Mingze and Qian, Lei and Mukherjee, Amitava and Xiang, Dongdong},
  journal={Quality Technology \& Quantitative Management},
  volume={21},
  number={5},
  pages={722--748},
  year={2024},
  publisher={Taylor \& Francis}
}

@article{ahmed2025comprehensive,
  title={Comprehensive review of high-dimensional monitoring methods: trends, insights, and interconnections},
  author={Ahmed, Fuhad and Mahmood, Tahir and Riaz, Muhammad and Abbas, Nasir},
  journal={Quality Technology \& Quantitative Management},
  volume={22},
  number={4},
  pages={727--751},
  year={2025},
  publisher={Taylor \& Francis}
}

@article{fang2022online,
  title={Online monitoring of high-dimensional binary data streams with application to extreme weather surveillance},
  author={Fang, Zhiwen and Li, Wendong and Liu, Xin and Pu, Xiaolong and Xiang, Dongdong},
  journal={Journal of Applied Statistics},
  volume={49},
  number={16},
  pages={4122--4136},
  year={2022},
  publisher={Taylor \& Francis}
}

@article{ye2022generic,
  title={A generic online nonparametric monitoring and sampling strategy for high-dimensional heterogeneous processes},
  author={Ye, Honghan and Liu, Kaibo},
  journal={IEEE Transactions on Automation Science and Engineering},
  volume={19},
  number={3},
  pages={1503--1516},
  year={2022},
  publisher={IEEE}
}

@article{ye2024online,
  title={Online monitoring of high-dimensional asynchronous and heterogeneous data streams for shifts in location and scale},
  author={Ye, Honghan and Zheng, Ziqian and Cheng, Jing-Ru C and Hable, Brock and Liu, Kaibo},
  journal={International Journal of Production Research},
  volume={62},
  number={3},
  pages={720--736},
  year={2024},
  publisher={Taylor \& Francis}
}

@book{maronna2019robust,
  title={Robust statistics: theory and methods (with R)},
  author={Maronna, Ricardo A and Martin, R Douglas and Yohai, Victor J and Salibi{\'a}n-Barrera, Mat{\'\i}as},
  year={2019},
  publisher={John Wiley \& Sons}
}

@article{shen2022robust,
  title={Robust tensor decomposition based background/foreground separation in noisy videos and its applications in additive manufacturing},
  author={Shen, Bo and Kamath, Rakesh R and Choo, Hahn and Kong, Zhenyu},
  journal={IEEE Transactions on Automation Science and Engineering},
  volume={20},
  number={1},
  pages={583--596},
  year={2022},
  publisher={IEEE}
}

@article{gomez2022adaptive,
  title={An adaptive sampling strategy for online monitoring and diagnosis of high-dimensional streaming data},
  author={G{\'o}mez, Ana Mar{\'\i}a Estrada and Li, Dan and Paynabar, Kamran},
  journal={Technometrics},
  volume={64},
  number={2},
  pages={253--269},
  year={2022},
  publisher={Taylor \& Francis}
}

@article{wang2022sgl,
  title={SGL-PCA: Health index construction with sensor sparsity and temporal monotonicity for mixed high-dimensional signals},
  author={Wang, Feng and Wang, Andi and Tang, Tao and Shi, Jianjun},
  journal={IEEE Transactions on Automation Science and Engineering},
  volume={20},
  number={1},
  pages={372--384},
  year={2022},
  publisher={IEEE}
}

@article{wang2025tensor,
  title={Tensor-on-tensor regression neural networks for process modeling with high-dimensional data},
  author={Wang, Qian and Bisheh, Mohammad N and Paynabar, Kamran},
  journal={arXiv preprint arXiv:2510.05329},
  year={2025}
}

\bigskip
\begin{center}
{\large\bf SUPPLEMENTARY MATERIAL }

\end{center}

\begin{description}

\item[Appendix A]:
\subsection{Proof of Proposition~\ref{prop:escp_soft}}
\label{app:proof_elcp_soft_threshold}

For fixed \(\mathcal L\), define
\[
    \mathcal R=\mathcal X-\mathcal L.
\]
The sparse-component update in \eqref{eq:escp_s_subproblem} is
\[
    \min_{\mathcal S}
    \left\{
        \|\mathcal R-\mathcal S\|_F^2
        +
        \alpha_1\|\mathcal S\|_1
    \right\}.
\]
Because both terms are separable across the entries of \(\mathcal S\), the
problem reduces to independently solving, for each
\((i,j,k)\),
\[
    \min_{s\in\mathbb R}
    \left\{
        (r-s)^2+\alpha_1|s|
    \right\},
\]
where \(r=R_{ijk}\).

The scalar objective is strictly convex because its quadratic term is strictly
convex. Hence, it admits a unique minimizer. Its subgradient optimality
condition is
\[
    0
    \in
    2(s-r)+\alpha_1\partial |s|.
\]

We consider the following three cases.

If \(s>0\), then \(\partial |s|=\{1\}\), and the optimality condition gives
\[
    2(s-r)+\alpha_1=0,
\]
so that
\[
    s=r-\frac{\alpha_1}{2}.
\]
This solution is positive only when
\[
    r>\frac{\alpha_1}{2}.
\]

If \(s<0\), then \(\partial |s|=\{-1\}\), and the optimality condition gives
\[
    2(s-r)-\alpha_1=0,
\]
so that
\[
    s=r+\frac{\alpha_1}{2}.
\]
This solution is negative only when
\[
    r<-\frac{\alpha_1}{2}.
\]

Finally, if \(s=0\), then
\(\partial |s|=[-1,1]\), and the optimality condition becomes
\[
    0
    \in
    -2r+\alpha_1[-1,1],
\]
which is equivalent to
\[
    |r|
    \leq
    \frac{\alpha_1}{2}.
\]

Combining the three cases yields
\[
    s^\star
    =
    \begin{cases}
        r-\dfrac{\alpha_1}{2},
        & r>\dfrac{\alpha_1}{2},
        \\[6pt]
        0,
        & |r|\leq\dfrac{\alpha_1}{2},
        \\[6pt]
        r+\dfrac{\alpha_1}{2},
        & r<-\dfrac{\alpha_1}{2},
    \end{cases}
\]
or, equivalently,
\[
    s^\star
    =
    \operatorname{sign}(r)
    \left(
        |r|-\frac{\alpha_1}{2}
    \right)_+.
\]

Applying this result entrywise gives
\[
    S_{ijk}^{\star}
    =
    \operatorname{sign}(R_{ijk})
    \left(
        |R_{ijk}|-\frac{\alpha_1}{2}
    \right)_+
    =
    \operatorname{soft}_{\alpha_1/2}(R_{ijk}),
\]
which proves Proposition~\ref{prop:escp_soft}.
\hfill\(\square\)

\subsection{Proof of Proposition~\ref{prop:fgcp_prox}}
\label{app:fgcp_prox_proof}

\begin{proof}
For fixed \(\mathcal L\), define the residual tensor
\[
    \mathcal R
    =
    \mathcal X-\mathcal L.
\]
The sparse-component subproblem in
\eqref{eq:fgcp_s_subproblem} is
\begin{equation}
\begin{aligned}
\min_{\mathcal S}
\quad
&
\|\mathcal R-\mathcal S\|_F^2
+
\alpha_1\gamma\|\mathcal S\|_1
\\
&\quad+
\alpha_1(1-\gamma)
\sum_{i=1}^{I}
\sum_{j=1}^{J}
\|\mathcal S_{ij:}\|_2.
\end{aligned}
\label{eq:fgcp_proof_tensor_problem}
\end{equation}
The temporal-factor penalty
\(\alpha_2\|\mathbf A^{(3)}\|_F^2\) does not appear in
\eqref{eq:fgcp_proof_tensor_problem} because it is constant with respect to
\(\mathcal S\) when \(\mathcal L\) is fixed.

Recall the threshold definitions
\[
    \tau_{\mathrm e}
    =
    \frac{\alpha_1\gamma}{2},
    \qquad
    \tau_{\mathrm g}
    =
    \frac{\alpha_1(1-\gamma)}{2}.
\]
Multiplying the objective in
\eqref{eq:fgcp_proof_tensor_problem} by the positive constant \(1/2\) does
not change its minimizer. Therefore, the problem is equivalent to
\begin{equation}
\min_{\mathcal S}
\left\{
    \frac{1}{2}
    \|\mathcal R-\mathcal S\|_F^2
    +
    \tau_{\mathrm e}\|\mathcal S\|_1
    +
    \tau_{\mathrm g}
    \sum_{i=1}^{I}
    \sum_{j=1}^{J}
    \|\mathcal S_{ij:}\|_2
\right\}.
\label{eq:fgcp_proof_scaled_problem}
\end{equation}

The objective in \eqref{eq:fgcp_proof_scaled_problem} can be written as
\begin{equation}
\sum_{i=1}^{I}
\sum_{j=1}^{J}
\left[
    \frac{1}{2}
    \|\mathcal R_{ij:}-\mathcal S_{ij:}\|_2^2
    +
    \tau_{\mathrm e}\|\mathcal S_{ij:}\|_1
    +
    \tau_{\mathrm g}\|\mathcal S_{ij:}\|_2
\right].
\label{eq:fgcp_proof_separable}
\end{equation}
Hence, the optimization is separable across the temporal fibers indexed by
\((i,j)\). It is therefore sufficient to solve the fiber-level problem
\begin{equation}
    \mathbf s^\star
    =
    \argmin_{\mathbf s\in\mathbb R^K}
    \left\{
        \frac{1}{2}\|\mathbf z-\mathbf s\|_2^2
        +
        \tau_{\mathrm e}\|\mathbf s\|_1
        +
        \tau_{\mathrm g}\|\mathbf s\|_2
    \right\},
\label{eq:fgcp_proof_fiber_problem}
\end{equation}
where
\[
    \mathbf z=\mathcal R_{ij:},
    \qquad
    \mathbf s=\mathcal S_{ij:}.
\]

The objective in \eqref{eq:fgcp_proof_fiber_problem} is strongly convex
because its quadratic term is strongly convex and both norm penalties are
convex. Consequently, it admits a unique minimizer. Its subgradient
optimality condition is
\begin{equation}
    \mathbf 0
    \in
    \mathbf s-\mathbf z
    +
    \tau_{\mathrm e}\partial\|\mathbf s\|_1
    +
    \tau_{\mathrm g}\partial\|\mathbf s\|_2.
\label{eq:fgcp_proof_optimality}
\end{equation}

Define the entrywise soft-thresholded vector
\begin{equation}
    \mathbf u
    =
    \operatorname{soft}_{\tau_{\mathrm e}}(\mathbf z).
\label{eq:fgcp_proof_u}
\end{equation}
By the optimality condition of the entrywise soft-thresholding operator,
there exists a vector \(\mathbf v\in\mathbb R^K\) such that
\begin{equation}
    \mathbf z
    =
    \mathbf u+\tau_{\mathrm e}\mathbf v,
    \qquad
    \mathbf v\in\partial\|\mathbf u\|_1.
\label{eq:fgcp_proof_soft_identity}
\end{equation}
In particular,
\[
    \|\mathbf v\|_\infty\leq 1.
\]

We now consider two cases.

\paragraph{Case 1: \(\|\mathbf u\|_2\leq\tau_{\mathrm g}\).}

We show that \(\mathbf s^\star=\mathbf 0\). Recall that
\[
    \partial\|\mathbf 0\|_1
    =
    \left\{
        \mathbf v\in\mathbb R^K:
        \|\mathbf v\|_\infty\leq1
    \right\}
\]
and
\[
    \partial\|\mathbf 0\|_2
    =
    \left\{
        \mathbf w\in\mathbb R^K:
        \|\mathbf w\|_2\leq1
    \right\}.
\]

If \(\tau_{\mathrm g}>0\), define
\[
    \mathbf w
    =
    \frac{\mathbf u}{\tau_{\mathrm g}}.
\]
Because
\(\|\mathbf u\|_2\leq\tau_{\mathrm g}\), we have
\(\|\mathbf w\|_2\leq1\), and hence
\(\mathbf w\in\partial\|\mathbf 0\|_2\). Furthermore,
\(\|\mathbf v\|_\infty\leq1\) implies
\(\mathbf v\in\partial\|\mathbf 0\|_1\). From
\eqref{eq:fgcp_proof_soft_identity},
\[
    \mathbf z
    =
    \tau_{\mathrm e}\mathbf v
    +
    \tau_{\mathrm g}\mathbf w.
\]
Therefore,
\[
    \mathbf 0
    \in
    -\mathbf z
    +
    \tau_{\mathrm e}\partial\|\mathbf 0\|_1
    +
    \tau_{\mathrm g}\partial\|\mathbf 0\|_2,
\]
which is the optimality condition
\eqref{eq:fgcp_proof_optimality} evaluated at
\(\mathbf s=\mathbf 0\).

If \(\tau_{\mathrm g}=0\), the condition
\(\|\mathbf u\|_2\leq\tau_{\mathrm g}\) implies
\(\mathbf u=\mathbf 0\). Equation
\eqref{eq:fgcp_proof_soft_identity} then gives
\[
    \mathbf z
    =
    \tau_{\mathrm e}\mathbf v,
\]
and the same optimality condition holds at
\(\mathbf s=\mathbf 0\).

Thus,
\begin{equation}
    \mathbf s^\star
    =
    \mathbf 0
    \qquad\text{when}\qquad
    \|\mathbf u\|_2\leq\tau_{\mathrm g}.
\label{eq:fgcp_proof_zero_case}
\end{equation}

\paragraph{Case 2: \(\|\mathbf u\|_2>\tau_{\mathrm g}\).}

Define
\begin{equation}
    c
    =
    1-\frac{\tau_{\mathrm g}}{\|\mathbf u\|_2},
    \qquad
    \mathbf s^\star
    =
    c\mathbf u.
\label{eq:fgcp_proof_nonzero_candidate}
\end{equation}
Because \(\|\mathbf u\|_2>\tau_{\mathrm g}\), we have \(c>0\).
Therefore, \(\mathbf s^\star\) and \(\mathbf u\) have the same support and
the same signs. It follows that the vector \(\mathbf v\) in
\eqref{eq:fgcp_proof_soft_identity} also satisfies
\[
    \mathbf v
    \in
    \partial\|\mathbf s^\star\|_1.
\]

Moreover,
\[
    \|\mathbf s^\star\|_2
    =
    c\|\mathbf u\|_2
    =
    \|\mathbf u\|_2-\tau_{\mathrm g},
\]
and
\[
    \frac{\mathbf s^\star}{\|\mathbf s^\star\|_2}
    =
    \frac{\mathbf u}{\|\mathbf u\|_2}.
\]
Consequently,
\begin{align}
    \mathbf s^\star
    +
    \tau_{\mathrm g}
    \frac{\mathbf s^\star}{\|\mathbf s^\star\|_2}
    &=
    \left(
        1-\frac{\tau_{\mathrm g}}{\|\mathbf u\|_2}
    \right)\mathbf u
    +
    \tau_{\mathrm g}
    \frac{\mathbf u}{\|\mathbf u\|_2}
    \nonumber\\
    &=
    \mathbf u.
\label{eq:fgcp_proof_group_identity}
\end{align}

Combining \eqref{eq:fgcp_proof_soft_identity} and
\eqref{eq:fgcp_proof_group_identity} yields
\[
    \mathbf z
    =
    \mathbf s^\star
    +
    \tau_{\mathrm g}
    \frac{\mathbf s^\star}{\|\mathbf s^\star\|_2}
    +
    \tau_{\mathrm e}\mathbf v.
\]
Equivalently,
\[
    \mathbf 0
    =
    \mathbf s^\star-\mathbf z
    +
    \tau_{\mathrm e}\mathbf v
    +
    \tau_{\mathrm g}
    \frac{\mathbf s^\star}{\|\mathbf s^\star\|_2}.
\]
Since
\[
    \mathbf v
    \in
    \partial\|\mathbf s^\star\|_1
\]
and
\[
    \frac{\mathbf s^\star}{\|\mathbf s^\star\|_2}
    \in
    \partial\|\mathbf s^\star\|_2,
\]
the candidate \(\mathbf s^\star\) satisfies
\eqref{eq:fgcp_proof_optimality}. Hence,
\begin{equation}
    \mathbf s^\star
    =
    \left(
        1-\frac{\tau_{\mathrm g}}{\|\mathbf u\|_2}
    \right)\mathbf u
    \qquad\text{when}\qquad
    \|\mathbf u\|_2>\tau_{\mathrm g}.
\label{eq:fgcp_proof_nonzero_case}
\end{equation}

Combining \eqref{eq:fgcp_proof_zero_case} and
\eqref{eq:fgcp_proof_nonzero_case} gives
\begin{equation}
    \mathbf s^\star
    =
    \begin{cases}
        \mathbf 0,
        &
        \|\mathbf u\|_2\leq\tau_{\mathrm g},
        \\[6pt]
        \displaystyle
        \left(
            1-\frac{\tau_{\mathrm g}}{\|\mathbf u\|_2}
        \right)\mathbf u,
        &
        \|\mathbf u\|_2>\tau_{\mathrm g},
    \end{cases}
\label{eq:fgcp_proof_piecewise_solution}
\end{equation}
where
\[
    \mathbf u
    =
    \operatorname{soft}_{\tau_{\mathrm e}}(\mathbf z).
\]
Equivalently,
\begin{equation}
    \mathbf s^\star
    =
    \left(
        1-
        \frac{\tau_{\mathrm g}}
        {
            \left\|
                \operatorname{soft}_{\tau_{\mathrm e}}(\mathbf z)
            \right\|_2
        }
    \right)_+
    \operatorname{soft}_{\tau_{\mathrm e}}(\mathbf z),
\label{eq:fgcp_proof_compact_solution}
\end{equation}
with the convention that the right-hand side is
\(\mathbf 0\) when
\(\operatorname{soft}_{\tau_{\mathrm e}}(\mathbf z)=\mathbf 0\).

Applying this solution independently to every residual fiber
\(\mathcal R_{ij:}\) yields
\[
    \mathcal S_{ij:}^{\star}
    =
    \left(
        1-
        \frac{\tau_{\mathrm g}}
        {\|\mathbf u_{ij}\|_2}
    \right)_+
    \mathbf u_{ij},
    \qquad
    \mathbf u_{ij}
    =
    \operatorname{soft}_{\tau_{\mathrm e}}
    \left(
        \mathcal R_{ij:}
    \right),
\]
which is the unique minimizer of
\eqref{eq:fgcp_s_subproblem}. This completes the proof.
\end{proof}

\subsection{Details of Simulation I}
\label{app:SimulationIDetails}

The correlation structure is defined over the domain \( T = [0, 1] \), with the Bessel correlation function of the first kind (Abramowitz and Stegun, 1964) used for the diagonal entries \( G_{ll}(s, t) \):
\[
G_{ll}(s, t) = (|s - t|^{\rho/2})^\nu \sum_{k=0}^{\infty} \frac{(-(|s - t|^\rho)^2/4)^k}{k! \Gamma(\nu + k + 1)},
\]
where \(\rho = 32\) and \(\nu = 0\). Eigenvalues \(\{\eta_k\}\) and eigenfunctions \(\{\vartheta_k\}\) are computed, and the off-diagonal cross-correlation \( G_{lj}(t_1, t_2) \) for \( l \neq j \) is given by:
\[
G_{lj} (t_1, t_2) = \sum_{k=1}^{\infty} \frac{\eta_k}{1 + |l - j|/\rho_{\text{cor}}} \vartheta_k (t_1) \vartheta_k (t_2), \quad \rho_{\text{cor}} = 1.
\]
The full covariance matrix \( G \) is spectrally decomposed to yield eigenvalues \(\{\lambda_i\}\) and eigenfunctions \(\{\psi_i\}\), truncated at \( L^* = 225 \) based on eigenvalue significance. Functional data is then synthesized as:
\[
Z(t) = \sum_{i=1}^{L^*} \xi_i \psi_i(t), \quad \xi_i \sim \mathcal{N}(0, \lambda_i),
\]
and combined with a mean function \( m(t) \), defined to simulate dynamic resistance curves (DRCs) in resistance spot welding:
\begin{align*}
m(t) &= 0.2074 + 0.3117 \exp(-371.4 h(t)) + \\ 
&\quad 0.5284(1 - \exp(0.8217 h(t))) \\
     &\quad - 423.3 \left[1 + \tanh(-26.15(h(t) + 0.1715))\right]
\end{align*}

where \( h(t) = (t - 0.0045)/(0.15 - 0.0045) \), as proposed in \cite{schwab2012improving}.

To simulate outliers, we define two contamination functions:
\[
CE(t) = 
\begin{cases} 
-M_E(t - 0.5) & \text{if } t < 0.5, \\
-2M_E(t - 0.5) & \text{if } t \ge 0.5,
\end{cases}
\]
\begin{align*}
CP(t) &= -m(t) - \left(\frac{M_P}{5}\right)t + 0.45\left(\frac{M_P}{5}\right) + 0.2074 \\
      &\quad + 0.3117 \exp(-371.4h(t)) + 0.5284(1 - \exp(0.8217h(t))) \\
      &\quad - 423.3 \left[1 + \tanh(-26.15(h(t) + 0.1715))\right],
\end{align*}
with \( h(t) = (0.0045 + t(0.15 - 0.0045))^{1+M_P/0.6} \). The final contaminated sample is given by:
\[
X(t) = m(t) + Z(t) \sigma + \varepsilon(t) + B_{\text{CaE}}CE(t) + B_{\text{CaP}}CP(t), \quad t \in T,
\]
where \( \sigma = 0.0125 \), \( \varepsilon(t) \sim \mathcal{N}(0, 0.0025^2) \), and \( B_{\text{CaE}}, B_{\text{CaP}} \sim \text{Bernoulli}(p_{\text{CaE}}, p_{\text{CaP}}) \) determine the activation of each contamination. These are applied to vectors \( CE = (B_{1,\text{CeE}}CE, \ldots, B_{p,\text{CeE}}CE)^\top \) and \( CP = (B_{1,\text{CeP}}CP, \ldots, B_{p,\text{CeP}}CP)^\top \), with independently sampled Bernoulli indicators. The functional data are discretely observed at 100 uniformly spaced time points. Both componentwise and casewise contamination scenarios are considered, affecting either individual or all signal channels per sample. Figure \ref{fig:Simulation data} shows an example of this simulation.

\vspace{5pt}

\subsection{Details of Simulation II}
\label{app:SimulationIIDetails}

Each simulated sample is represented by a matrix
\(\mathcal X_{0,i::}\in\mathbb R^{p\times T}\), where \(p\) is the number of
functional channels and \(T\) is the number of discretization points. To
generate the nominal profiles, let
\(\mathbf W_i\in\mathbb R^{T\times p}\) contain \(p\) independent Gaussian
process realizations with covariance function
\begin{equation}
    \Sigma(q,q')
    =
    \exp\left\{
        -\left|2(q-q')\right|^2
    \right\}.
    \label{eq:sim2_gp_kernel}
\end{equation}
Cross-channel dependence is introduced through the compound-symmetry
covariance matrix
\begin{equation}
    \mathbf C
    =
    \begin{bmatrix}
        1      & \rho_c & \cdots & \rho_c \\
        \rho_c & 1      & \cdots & \rho_c \\
        \vdots & \vdots & \ddots & \vdots \\
        \rho_c & \rho_c & \cdots & 1
    \end{bmatrix}
    =
    \boldsymbol\Delta\boldsymbol\Delta^\top.
    \label{eq:sim2_channel_covariance}
\end{equation}
The nominal multichannel observation is then generated as
\begin{equation}
    \mathcal X_{0,i::}
    =
    \left(
        \mathbf W_i\boldsymbol\Delta^\top
    \right)^\top.
    \label{eq:sim2_nominal_sample}
\end{equation}
We use \(N=200\), \(p=5\), \(T=100\), and
\(\rho_c=0.75\).

Trajectory-level masking is introduced independently for each channel. For
each \(j=1,\ldots,p\), \(5\%\) of the samples are selected without
replacement, and their complete temporal trajectories in channel \(j\) are
set to zero. Let
\(\mathcal M\in\{0,1\}^{N\times p\times T}\) denote the resulting mask, where
\(\mathcal M_{ij:}=\mathbf 0\) for a selected sample--channel pair and
\(\mathcal M_{ij:}=\mathbf 1\) otherwise. Thus, the masking occurs at the
sample--channel fiber level rather than independently at individual tensor
entries.

Synthetic anomalies are introduced into \(5\%\) of the samples. For each
selected anomalous sample \(i\), a channel-loading vector
\(\mathbf u_i\in\mathbb R^p\) is generated from a standard multivariate
Gaussian distribution and normalized such that
\(\|\mathbf u_i\|_2=1\). A temporal mean-shift pattern
\(\mathbf v_i(\delta)\in\mathbb R^T\) is defined as
\begin{equation}
    v_{ik}(\delta)
    =
    \begin{cases}
        0,
        & k < \kappa_i,\\
        \delta,
        & k \geq \kappa_i,
    \end{cases}
    \label{eq:sim2_temporal_anomaly}
\end{equation}
where the change point \(\kappa_i\) is selected randomly after approximately
the first third of the discretization domain. The anomaly matrix for sample
\(i\) is
\begin{equation}
    \mathcal A_{i::}(\delta)
    =
    \mathbf u_i\mathbf v_i(\delta)^\top.
    \label{eq:sim2_rankone_anomaly}
\end{equation}
This construction produces a rank-one anomaly across the channel and temporal
modes. All five channels share the same temporal anomaly support, while their
shift magnitudes and directions are determined by the entries of
\(\mathbf u_i\). For samples not selected as anomalous,
\(\mathcal A_{i::}(\delta)=\mathbf 0\).

The anomaly is added before the trajectory mask is applied. Additive symmetric
perturbations are subsequently introduced, so that the tensor supplied to the
anomaly-detection methods is
\begin{equation}
    \mathcal X
    =
    \left(
        \mathcal X_0+\mathcal A(\delta)
    \right)
    \odot\mathcal M
    +
    \mathcal E,
    \label{eq:sim2_observed_tensor}
\end{equation}
where \(\odot\) denotes elementwise multiplication. In the numerical
experiment, \(98\%\) of the entries are selected without replacement and
perturbed by either \(+a\) or \(-a\), with equal proportions, where
\begin{equation}
    a
    =
    0.8
    \sqrt{\frac{0.92}{0.98}}.
    \label{eq:sim2_noise_amplitude}
\end{equation}
The results reported in Table~\ref{tab:simulation2_comparison} correspond to
\[
    \delta
    \in
    \{0.25,\,0.50,\,1.00,\,2.00,\,3.00\}.
\]
A sample-level anomaly label is assigned according to whether that sample was
selected for anomaly injection; the binary labels are not interpreted as the
exact entrywise support of the anomaly.

\end{description}

\end{document}